\documentclass[review,3p,11pt]{article}

\usepackage{lipsum}
\usepackage{graphicx}
\usepackage{algorithmic}

\usepackage{amsopn}
\usepackage{mathtools}

\makeatletter
\newcommand*{\addFileDependency}[1]{% argument=file name and extension
  \typeout{(#1)}% latexmk will find this if $recorder=0 (however, in that case, it will ignore #1 if it is a .aux or .pdf file etc and it exists! if it doesn't exist, it will appear in the list of dependents regardless)
  \@addtofilelist{#1}% if you want it to appear in \listfiles, not really necessary and latexmk doesn't use this
  \IfFileExists{#1}{}{\typeout{No file #1.}}% latexmk will find this message if #1 doesn't exist (yet)
}
\makeatother

\usepackage{amsthm}
\usepackage{amsmath,amsfonts,amssymb}
\newcommand{\prox}{\operatorname{prox}}
\newcommand{\grad}{\nabla}

\DeclareMathOperator*{\argmin}{argmin}

\newcommand{\dx}{\operatorname{d}\!x}
\newcommand{\dtau}{\operatorname{d}\!\tau}

\newcommand{\la}{\lambda}

\newcommand{\norm}[1]{\lVert#1\rVert}
\renewcommand{\div}{\operatorname{div}}
\usepackage{hyperref}

\newtheorem{remark}{Remark}
\newtheorem{example}{Example}
\newtheorem{theorem}{Theorem}

\usepackage[inkscapearea=page]{svg}
\usepackage{algorithm2e}
\newcommand{\U}{\mathbb A}
\newcommand{\Hr}{\mathbb H}

\newcommand{\pdF}{\mathcal{F}}

\newcommand{\uu}{\boldsymbol{u}}
\newcommand{\vv}{\boldsymbol{v}}

\newcommand{\Do}{\mathbf{D}}

\newtheorem{problem}[theorem]{Problem}

\usepackage{bbm}
\ifpdf
  \DeclareGraphicsExtensions{.eps,.pdf,.png,.jpg}
\else
  \DeclareGraphicsExtensions{.eps}

\fi
\begin{document}
%
% paper title
% Titles are generally capitalized except for words such as a, an, and, as,
% at, but, by, for, in, nor, of, on, or, the, to and up, which are usually
% not capitalized unless they are the first or last word of the title.
% Linebreaks \\ can be used within to get better formatting as desired.
% Do not put math or special symbols in the title.
\title{ Self2Seg: Single-Image Self-Supervised Joint Segmentation and Denoising}
%\title{SLS: When squirrels love sheeps, Denoising Fused with Segmentation, A Single-Image Self-Supervised \\ Joint Segmentation and Denoising}
\author{{{Nadja Gruber \thanks{Corresponding author: nadja.gruber@uibk.ac.at}} $^,$\thanks{Department of Mathematics, University of Innsbruck, Austria}} $^,$\thanks{VASCage-Research Centre on Vascular Ageing and Stroke, Innsbruck, Austria} 
\and
J. Schwab \thanks{MRC Laboratory of Molecular Biology, Cambridge, UK}
\and N. Debroux \thanks{Institut Pascal, Université Clermont Auvergne, Clermont-Ferrand, France}
\and N. Papadakis \thanks{Institut de Mathématiques de Bordeaux, Bordeaux, France}
\and M. Haltmeier \footnotemark[2]}

\maketitle

% As a general rule, do not put math, special symbols or citations

% For peer review papers, you can put extra information on the cover
% page as needed:
% \ifCLASSOPTIONpeerreview
% \begin{center} \bfseries EDICS Category: 3-BBND \end{center}
% \fi
%
% For peerreview papers, this IEEEtran command inserts a page break and
% creates the second title. It will be ignored for other modes.
\maketitle

\begin{abstract}
We develop Self2Seg, a self-supervised method for the joint segmentation and denoising of a single image. To this end, we combine the advantages of variational segmentation with self-supervised deep learning. One major benefit of our method lies in the fact, that in contrast to data-driven methods, where huge amounts of labeled samples are necessary, Self2Seg segments  an image into meaningful regions without any training database.  Moreover, we demonstrate that self-supervised denoising  itself is significantly improved through the region-specific learning of Self2Seg.
  Therefore, we introduce a novel self-supervised energy functional in which denoising and segmentation are coupled in a way that both tasks benefit from each other. We propose a unified optimisation strategy and numerically show that for noisy microscopy images our proposed joint approach outperforms its sequential counterpart as well as alternative methods focused purely on denoising or segmentation. 
\end{abstract}

% REQUIRED

\section{Introduction}\label{sec:intro}
Image denoising and segmentation are fundamental problems in image processing~\cite{scherzer2009variational, buchholz2021denoiseg, kollem2019review}. Especially in many biomedical applications, such as fluorescence microscopy or transmission electron microscopy, one is interested in the segmentation of objects. Training data for this task is typically scarce and hard to obtain due to the intrinsic complexity and high noise of such images, as well as the long time required by experts to label them. Therefore, there is a need for unsupervised methods for tackling the two imaging tasks in a unified way. In this work, we propose such a framework  by beneficially  combining  and  extending ideas of the Chan-Verse segmentation model \cite{chan1999active,gruber2024lifting} and self supervised learning  \cite{lequyer2022fast,krull2019noise2void}. 

As we demonstrate in this paper joint self-supervised denoising
and segmentation can benefit a lot from each other. By identifying individual regions, segmentation guides the denoising process to specific image structures. On the other hand, the adaptation of a denoiser to a specific region successfully guides the segmentation process.

%and apply it to a subset of a popular, publicly available dataset of microscopy images.

%The objective of segmentation is to divide a given image into different, meaningful regions, while denoising describes the task of removing noise from a corrupted image. The main difficulty in noise removal is to flatten the unwanted high-frequency corruption while preserving essential features such as edges. At first glance, denoising and segmentation are two different applications. Nevertheless, both tasks share relationships, as very similar models can be used to solve both problems~\cite{cai2019linkage}. As we demonstrate in this work, denoising and segmentation can benefit a lot from each other. By identifying edges, segmentation guides the denoising process to preserve sharp structures while smoothing the unwanted high-frequency residuals. Also, by removing unnecessary and misleading information from images, denoising helps and improves the segmentation accuracy.

There exist at least two main kinds of approaches to tackle both tasks individually. The first class of methods involves the minimization of an energy functional within graph or variational frameworks. The second type of approaches that recently became popular considers deep learning techniques, especially those based on convolutional neural networks~\cite{litjens2017survey}.

Before we start reviewing related works tackling the task of denoising and segmentation in a unified way, we provide the motivation behind the proposed strategy.

\subsection{Motivation and Contributions}\label{sec:contributions}
To overcome limitations of existing algorithms (see ~\ref{rel_work}), in this paper we introduce a novel joint denoising and segmentation algorithm combining the  strengths of variational models and deep learning. 

We employ toy examples to demonstrate how segmentation effectively guides the denoising process. First, we generate a 256$\times$256 image featuring stripe patterns oriented differently across various regions, as depicted in Figure~\ref{fig:toy}. Gaussian noise is manually added to corrupt the image. We utilize two linear neural networks, each serving as an "expert" dedicated to either the foreground or background. These networks consist of a single convolutional layer with a $15\times 15$ filter, trained using a slightly modified version of the Noise2Fast strategy~\cite{lequyer2022fast}, detailed in Section~\ref{sec:numerical}. Specifically, we confine network training to the respective image regions by masking the loss function and restricting training to $30\times 30$ boxes, as illustrated in Figure~\ref{fig:toy}. Not surprisingly, we observe that the learned filters adapt to the signal structure within their designated regions, leading to higher error patterns in untrained areas. This observation forms the basis for leveraging region-specific denoising for segmentation. Further experimental details for this toy example are available in Section~\ref{sec:experiments}.

%At the same time if we use a single filter for the whole region, we see that we obtain a mixture of these filters which is basically a mean filter. 

\begin{figure}[ht]
    \centering   \includegraphics[width = 0.8\columnwidth]{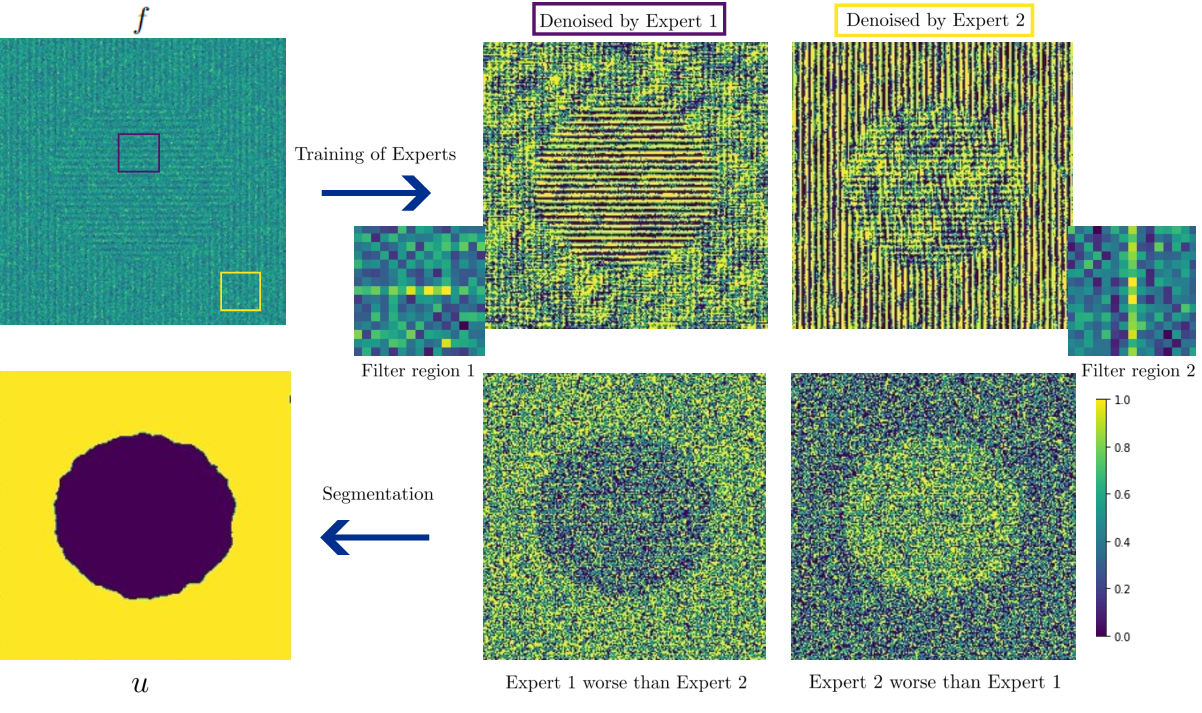}
    \caption{Visualisation of the idea behind the proposed joint denoising and segmentation model. Here, we trained two networks consisting of one single filter using the Noise2Fast~\cite{lehtinen2018noise2noise} strategy and restricted the training to the two boxes marked in the noisy image $f$. From the two right binary images in the bottom row, we observe that the two denoising experts perform much better in the  region they have been trained on. The difference images (noisy image minus Denoised by Expert 1 (resp. 2) can then be used in the segmentation process, by exploiting the fact that regions with a small denoising error for the first (resp. second) expert can be assigned as foreground (resp. background). }
    \label{fig:toy}
\end{figure}

\begin{figure}[ht]
    \centering
    \includegraphics[width = 0.93\columnwidth]{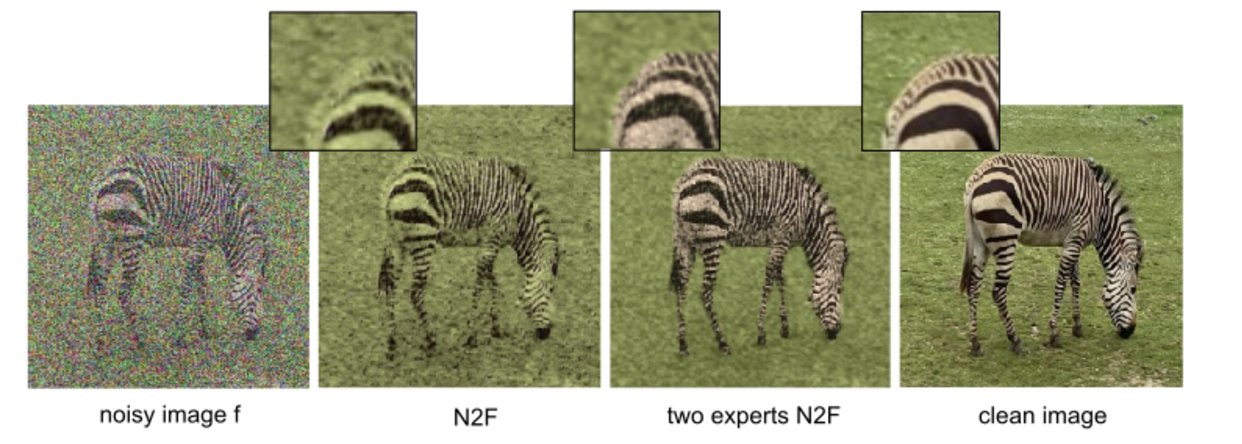}
    \caption{Given noisy RGB input image (corrupted with Gaussian noise, noise level = 0.75), denoised image using Noise2Fast on the whole image, region-specific experts, and ground truth image. We clearly observe sharper edges, and better recovered color information in the ``two-experts''-example.}
    \label{fig:zeb}
\end{figure}
The positive effect of segmentation on the denoising process is even more evident in the natural image shown in Figure~\ref{fig:zeb}. We used the network architecture proposed by the authors in~\cite{lequyer2022fast}, resulting in a non-linear neural network. First, the denoising network was trained and subsequently applied to the whole image. The second image shows the result obtained with two separately trained neural networks. This strategy yields a better visual result, which is further confirmed by PSNR values of 19.69 and 19.12, respectively. The noisy image has been generated by scaling the given clean RGB input image to $[0,1]$, and adding randomly distributed gaussian noise scaled with the maximum pixel value, ensuring that the noise is proportional to the image intensity. Here, we used a manually generated mask of the zebra and background region, and during training, computed the mean squared error (MSE) restricted to the two different regions, respectively.

More specifically, we minimize  the loss function 
\begin{equation} \label{eq:joint-intro}
\begin{aligned}
\mathcal {E}_{f,\lambda} (u, \Do_F,  \Do_B) 
= \lambda\lvert u\rvert_{\text{TV}}& + \int_\Omega \left(f(x) - \Do_F(f) \right)^2 u(x) dx  
 \\ &\hspace{0.1\columnwidth}+\int_\Omega \left(f(x) - \Do_B(f)(x)\right)^2(1-u(x))dx\, 
\end{aligned}
\end{equation} 
jointly over  $u$, $\Do_F$ and  $\Do_B$, taken from a specific class $\mathcal{U}\times\mathcal{D}_F\times\mathcal{D}_B$. Here $u$ and $1-u$ take values in $[0,1]$ and define  the  foreground and background regions and $\Do_F$ and $\Do_B$ are denoisers that should adapt to the two regions. Coupling a variational segmentation model with a self-supervised denoising CNN, overcomes  the need for labeled data and pre-training. We integrate denoising and segmentation tasks to mutually enhance performance.

Specifically, we develop dedicated denoising networks for foreground and background regions, respectively. In that way, the overall denoising performance is improved by exploiting performance differences between the regions to derive the segmentation mask.

Conversely, we demonstrate that existing self-supervised denoising methods such as Noise2Self~\cite{bae2017convex} or Noise2Fast~\cite{lequyer2022fast}  benefit from splitting up the image into regions by exploiting different structures and performing region-specific learning.

 Unlike most previous deep learning methods for segmentation, our approach  is self-supervised  and achieves comparable results using just a single image. In contrast to existing approaches, we actually exploit the  performance difference of denoisers which can vary across the different regions in the image, and show that by the way our approach is designed, both denoising and segmentation tasks have a positive influence on each other.

\subsection{Related work}
\label{rel_work}
\subsubsection{Joint denoising and segmentation}
In order to improve noise robustness and handle intensity inhomogeneities in the Chan-Vese model, several local region-based extensions have been proposed in \cite{li2007implicit, zhang2010active, niu2017robust}. Nevertheless, these methods remain sensitive to manually crafted features and the initial contour. Pre-filtering tools to better prepare the image for segmentation are considered in \cite{cai2013two, li2020three, zhan2013improved}. The work by \cite{liu2014signal} proposes a segmentation-based image denoising algorithm for signal-dependent noise. After initial denoising, segmentation is applied to the pre-filtered image, and for each segment, the noise level is estimated. Subsequently, for each region, a separate denoiser is applied. In \cite{cai2015variational}, a model tackling the segmentation of images with a high level of noise or blurriness is presented. To this end, they propose a variational approach, coupling an extension of the piecewise constant Mumford-Shah model~\cite{mumford1989optimal} with an image restoration model, making it more robust in processing the given corrupted image. A variational approach for joint reconstruction and segmentation is proposed in \cite{corona2019enhancing}, where a model consisting of a total variation regularized reconstruction from undersampled data and a Chan-Vese-based segmentation is used. The authors show the improvement of the joint reconstruction and segmentation performance compared to the sequential approach. In another paper by \cite{ramlau2007mumford}, the Mumford-Shah level set method is shown to improve the quality of reconstructed images and the accuracy of segmentation.

\subsubsection{Self-supervised image denoising}\label{sec:self-super}

In recent times, self-supervised deep learning techniques got increasing attention in the field of machine learning.

In~\cite{ulyanov2018deep}, Ulyanov et al. exploit the fact that the internal structure of CNNs inherently resonates with the distribution of natural images, and utilize this observation for image restoration without the need for additional training data. For each single image to restore, this method thus proposes to train a CNN to reconstruct the considered image. The idea is that early stopping of the training allows for the recovery of a regularized, denoised image. A different strategy is proposed in Noise2Noise~\cite{lehtinen2018noise2noise}, where noisy image pairs are mapped to one another. The drawback of the latter method is that it still relies on the availability of such pairs. In practice, even acquiring two noisy realizations of the same image content is often difficult~\cite{buchholz2021denoiseg}. To this end, self-supervised training methods operating on one single noisy image, such as Noise2Void~\cite{krull2019noise2void}, Noise2Self~\cite{batson2019noise2self}, and more recently, Noise2Fast~\cite{lequyer2022fast}, have been proposed as promising alternatives. The self-supervision is accomplished by excluding/masking the center (blind spot) of the receptive field of the network. In this type of training, it is assumed that the noise is pixelwise independent and that the true intensity of a pixel can be predicted from the local image context, with the exception of the blind spots~\cite{krull2019noise2void}. As one ingredient of our method we utilized the training strategy adapted from Noise2Fast~\cite{lequyer2022fast}, which is a variant of Noise2Self/Noise2Void. Noise2Fast offers the advantages of computational efficiency while maintaining strong performance. The method itself will be explained in more detail in Sections~\ref{sec:problem} and~\ref{sec:numerical}.

\subsubsection{Learning methods for denoising and segmentation}

In the context of microscopy data, purely deep learning-based approaches dealing with both segmentation and denoising are proposed in \cite{prakash2020leveraging, buchholz2021denoiseg}. In \cite{prakash2020leveraging}, it is demonstrated on various microscopy datasets that the use of self-supervised denoising priors improves the segmentation results, especially when only a few ground truth segmentation masks are available for training. In \cite{buchholz2021denoiseg}, the authors propose \textsc{DenoiSeg}, consisting of a U-Net for image segmentation and the self-supervised denoising scheme Noise2Void~\cite{krull2019noise2void}, which are combined and trained with a common loss. It is shown that global optimization outperforms the sequential counterpart, where the image is denoised first and then segmented. To reach high denoising performance, a huge amount of noisy data is required. Moreover, the loss function is just the sum of the segmentation and denoising losses. The fusion of the two tasks is achieved by sharing the network; the loss function does not incorporate any coupling between the two tasks. A slightly different approach is introduced in \cite{liu2017image}. Nevertheless, this method requires ground truth data for both imaging tasks.

The paper is organized as follows: In Section \ref{sec:problem}, we formulate the problems and review algorithmic components. Section \ref{sec:jointmodel} details our proposed algorithm, followed by numerical implementation in Section \ref{sec:numerical}. We apply the method to microscopy data in Section \ref{sec:experiments} and demonstrate its versatility on natural images with minimal user guidance. Finally, we conclude and discuss avenues for future research.

\section{Problem Description}\label{sec:problem}

We will now fix the notation, formalize the problem, and describe the main ingredients that are used for the proposed unified denoising and segmentation method.

In the following, we denote by $\Omega\subset\mathbb{R}^2$ a bounded set with Lipschitz boundary, and by $\mathbb{F}$ a space of functions $f\colon\Omega\rightarrow\mathbb{R}^d$, with $d=1$ for grayscale images, and $d=3$ in the RGB case. We consider a given noisy image $f\in \mathbb{F}$, which we want to jointly denoise, and split up into $C$ different regions.

\begin{problem}[Image Denoising]
The goal of image denoising is to recover a clean image $g$ from a noisy observation $f$ which follows an image degradation model $f = g + n$, where $n$ is the signal-degrading noise that we want to remove.
\end{problem}

Note that although other degradation types are possible, we assume an additive model here, and specifically, we will consider random noise with an expected value of zero.

\begin{problem}[Image Segmentation]
Image segmentation refers to the process of automatically dividing an image into meaningful regions. Based on specific characteristics of a given image $f\in\mathbb{F}$, one is interested in splitting the image domain into two (in the case of binary segmentation) regions $\Sigma$ and $\Omega\setminus\Sigma$. In the case of multiclass segmentation, the objective is to build a partition $\Omega=\bigcup_{i=1}^C\Sigma_i$ of the image domain into $C$ disjoint regions (classes), where each of the regions $\Sigma_1,\dots, \Sigma_{C-1}$ represents a specific structure of objects in $f$, and $\Omega\setminus(\Sigma_1\uplus \Sigma_2\uplus\dots \uplus\Sigma_{C-1})$ represents the background.
\end{problem}

In this work, we address these two problems simultaneously by designing an energy functional in a way that both tasks benefit from each other. Next, we discuss the two main components that form the basis of our approach.

\subsection{Convex Chan-Vese Formulation}
In~\cite{chan2006algorithms}, Chan et al propose to relax the binary Chan-Vese segmentation problem and let the desired solution $u$ take values in $[0,1]$. The resulting convex energy is 
\begin{align}\label{nikolova}
    \min_{0\leq u\leq 1}\int_{\Omega}\lvert \nabla u(x)\rvert + \lambda\left(\int_{\Omega}(c_1 - f(x))^2 u(x)\dx + \int_\Omega (c_2 - f(x))^2(1-u(x))\dx\right).
\end{align}
The authors showed that, for any fixed constants $c_1, c_2\in\mathbb{R},$ a global minimiser for the non-convex problem can be found by carrying out the minimisation in~\eqref{nikolova}, and setting $\Sigma = \{x:u(x)>\tau\}$ for a.e. $\tau\in[0,1]$.

Though the model is convex %\np{This is wrong the model is not strictly convex so uniqueness of the minimiser is not guarantee}, and we can conclude that every minimiser which we obtain during the optimisation process is unique, the model 
it still suffers from difficulties in segmenting images where the piecewise constant assumption is not a suitable prior for the different regions in the image, or if the image is corrupted by severe noise. These issues are the main problems to solve in the current paper by substituting $c_1$ and $c_2$ by more flexible neural network denoisers.

\subsection{Self-supervised single-image based denoising}
Before presenting the self-supervised loss, we introduce a very important property of denoising functions we are interested in.

Formulated in a very general way, for a given noisy image $f = g+n$, self-supervised denoising methods are based on some variant of the self supervised loss
\begin{align}\label{eq:denoisloss}
    \mathcal{L}_f(\Do) = \int_{\Omega} (\Do(f)(x)-f(x))^2\dx,\, \Do\in\mathcal{D}
\end{align}
minimised with an optimisation strategy, that excludes the identity function, $\Do = \rm{Id}$, which would be the exact minimiser but does not yield a denoised image.

One strategy to overcome this problem is the method introduced in \cite{ulyanov2018deep} where a generative model $\Do$ for minimising \eqref{eq:denoisloss} is combined with early stopping. In this framework, the convolutional structure and early stopping prevents $\Do$ to learn the fine image features (noise) to obtain a denoised image. 
Another strategy is linear filtering with a restriction on the filter~\cite{batson2019noise2self, krull2019noise2void}. For example, a filter which is zero in its central position, and therefore not taking into account the information of this pixel but only the surrounding areas can be used to denoise an image minimising \eqref{eq:denoisloss}.
Another type of method tries to minimise a slightly different functional. Motivated by the previous example, the authors of~\cite{lequyer2022fast} introduce $N$ random binary masks $\mathcal{H}_k$ that delete information in the image. Training is then done using the loss function
\begin{align}\label{eq:n2f_loss}
    \mathcal{L}_f(\Do) = \frac{1}{N}\sum_{k=1}^N\int_\Omega (\-\Do(\mathcal{H}_k\cdot f)(x)-f(x))^2\cdot(1-\mathcal{H}_k)\dx.
\end{align}
This training strategy prevents the network from learning the identity operator. In this work, we use a variant of Noise2Fast~\cite{lequyer2022fast}, which is numerically cheaper than Noise2Void or Noise2Self, and where $\mathcal{H}_k, k=1,\dots, N$, consist of even and odd pixels of the image, respectively.

\section{Proposed Joint Denoising and Segmentation}\label{sec:jointmodel}
We now introduce our joint model, inspired by the observations described in Section~\ref{sec:contributions}. To control binary segmentation, we propose to train two denoising neural networks, each focusing on performing well in one of the regions to be segmented (cf. Figure \ref{fig:toy} and \ref{fig:zeb}). We denote these ``experts'' by $\Do_F$ for the foreground, and $\Do_B$ for the background. These experts are neural networks with parameters $F$ and $B$, that are trained with a modified  denoising strategy. %The corresponding network parameters are denoted by $F$ and $B$, respectively. 
Let us mention that the model is presented in the case of two regions, but the extension to multi-class is straightforward, following for instance the framework in~\cite{gruber2024lifting,bae2017convex, mevenkamp2016variational}

%In Section~\ref{sec:proposed}, we present the proposed joint energy function designed for the combined denoising and segmentation process. This energy generalizes the convex Chan-Vese functional~\eqref{nikolova} with a data-fidelity term defined from the self-supervised denoising method . The optimisation scheme is performed in an alternating way, as presented in Section~\ref{jointopti}. We finally provide theoretical convergence results for our algorithm in Section~\ref{sec:theoretical}.

\subsection{Joint energy functional}\label{sec:proposed}
 In the following, we denote by $\mathrm{BV}(\Omega)$ the space of all integrable functions $u:\Omega\rightarrow\mathbb{R}$ with bounded total variation $|u|_{\text{TV}}$, and consider the  admissible set
\begin{align*}
    \U\coloneqq\{u\in \mathrm{BV}(\Omega)\mid 0\leq u \leq 1\}.
\end{align*}
Further, let $i_{\U}:\mathrm{BV}(\Omega)\rightarrow[0,\infty]$ denote the associated indicator function, which is 0 inside $\U$, and $\infty$ elsewhere. 
The parameters of the two denoising experts, $\Do_F$ and  $\Do_B$ are denoted by $\boldsymbol{W}=(F,B)\in\mathbb{R}^{L_1\times L_2},$ and are respectively dedicated to the foreground and the background. Note that the number of parameters for $\Do_F$ and $\Do_B$ do not necessarily have to be the same. These two experts are neural networks trained using the strategy proposed in~\cite{lequyer2022fast}. We consider the joint model
\begin{equation} \label{eq:joint2}
\begin{split}
\mathcal {E}_{f,\lambda}(u, \Do_F, \Do_B) 
= i_{\U}(u) + \lambda\lvert u\rvert_{\text{TV}} &+ \int_\Omega \left(f(x)-\Do_F(f)(x)\right)^2 u(x) \dx \\ 
&+\int_\Omega \left(f(x)-\Do_B(f)(x)\right)^2(1-u(x))\dx\, .
\end{split}
\end{equation}
Note that for fixed network parameters $\boldsymbol{W}$, the proposed energy is convex in $u$.
%This model is convex in $u$, and in $\theta$. 
Moreover, we can threshold the result and still have a global optimum (see Theorem~\ref{thresholding}).

Figure~\ref{example_zebri} illustrates the idea behind the proposed segmentation model. For greyscale images, one can initialise the algorithm by thresholding image values. In more complex cases, a user can be asked to provide representative boxes for the background and foreground regions. Then, alternately, the denoising experts are trained on subsets of the two different segmented regions and the segmentations are updated. In practice, the data fidelity term in~\eqref{eq:joint2} is updated given the denoising performance of the two experts $\Do_F$ and $\Do_B$. For fixed network parameters $\boldsymbol{W}$, the energy~\eqref{eq:joint2} is minimised. Repeating this procedure until a convergence criteria is met, we obtain the segmentation mask $u$, as well as the denoised image $g\approx u\odot \Do_F(f) + (1-u)\odot\Do_B(f)$.
%\textcolor{blue}{This example demonstrates, that due to way  is designed, our approach can also handle noise-free data. This is an advantage compared to the application of another variational denoising method, where in such a case, the identity would be learned.}
\begin{figure}[ht]
    \centering
    \includegraphics[width = 0.93\columnwidth]{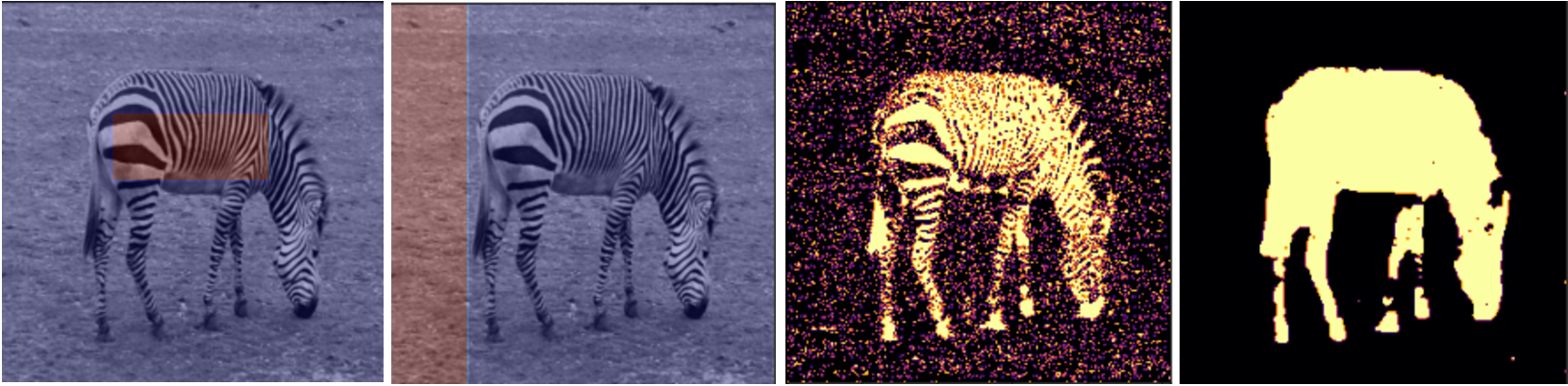}
    \caption{The first image shows the given greyscale input image $f$, and user defined boxes representing rough foreground and background regions. The third image highlights pixels where the foreground expert denoiser performs better than the background one, while the last image is the segmentation result obtained by minimising the proposed energy~\eqref{eq:joint2}.}
    \label{example_zebri}
\end{figure}

\begin{example}
Here, we give examples for neural networks that act as denoisers and relate to existing approaches. 
\begin{itemize}
    \item\textbf{Constant Background:}
In case where the background is assumed constant, one could simply assume that $\Do_B(f)=B\odot\mathbbm{1}$, which corresponds to estimate a scalar value $B$ being the mean value of the given image inside the corresponding region as in the original Chan and Vese model.
    \item\textbf{Linear filter:}
    In this case, the network is linear with respect to the network parameters $B$, more precisely,
$\Do_B(f)=\omega_{B}\ast f$, leading to a bi-convex energy functional~\eqref{eq:joint2}. In our toy example in Figure~\ref{fig:toy}, we have applied such a linear network consisting of one single filter of kernel size $15\times 15$.
\item\textbf{Filtering of data fidelity term:}
    When one of the regions is assumed to be constant and high noise levels are present, mean filtering improves the results. The data fidelity terms of energy~\eqref{eq:joint2} can  then be  replaced by $\int_{\Omega}\left[K_{\sigma}\ast\left(f-\Do_F(f)\right)\right] ^2u$ and $\int_{\Omega}\left[K_{\sigma}\ast\left(f-\Do_B(f)\right)\right] ^2(1-u)$, respectively, where  $K_\sigma$ is a mean filter with kernel size $\sigma$.
    A similar approach has been done in~\cite{li2007implicit}, where a more robust version of the Chan-Vese model~\cite{chan1999active} has been proposed.
    %by introducing Gaussian convolution in the data fitting terms, in order to make the method robust to non homogeneous regions. 

\item\textbf{Generic CNN:}
   Any typical state of the art denoising neural network (Deep image prior~\cite{ulyanov2018deep}, Noise2Void~\cite{krull2019noise2void}) can be used in our framework. Note, that in this case the bi-convexity of energy~\eqref{eq:joint2} is not ensured anymore.
\end{itemize}
\end{example}
In the next paragraph, we discuss in more detail the joint alternating optimisation procedure we propose to minimise energy~\eqref{eq:joint2}.
\subsection{Joint optimisation}\label{jointopti}
We propose to iteratively optimise problem~\eqref{eq:joint2} with an alternating procedure~\cite{csiszar1984information}. In case the denoising step does not exactly minimise energy~\eqref{eq:joint2}, we actually alternate between minimising two slightly different functionals. For the sake of readability, this is not indicated in the notation. 
We start with the initialisation of the segmentation mask $u.$ This is either achieved by thresholding for greyscale images, or as shown in Figure~\ref{example_zebri}, manually choosing boxes representing the different regions to segment in the image. Then, based on the initial guess, the denoising expert(s) $\Do_F,$ and $\Do_B$ are trained on the given initial masks. To this end, we use the ADAM optimiser~\cite{kingma2014adam} until convergence.
As a next step, for fixed network parameters $\boldsymbol{W}=(F,B)$, we update the segmentation mask $u$. For fixed $\boldsymbol{W},$ the energy functional~\eqref{eq:joint2} is convex, and all the necessary assumptions for the application of the primal dual algorithm~\cite{chambolle2011first} are fulfilled. A more detailed description on the considered discrete schemes is provided in Section~\ref{sec:numerical} (see Algorithm~\ref{alg2}). These alternate steps are repeated as long as the decrease of energy~\eqref{eq:joint2} is greater than $15$ percent (i.e. $p=0.15)$, which we empirically found to give a good compromise between computation speed and quality of the results. 

The overall joint optimisation scheme is presented in Algorithm~\ref{algo_alternate}. For both denoising experts, we use the same network architecture, namely a convolutional neural network consisting of four sequential instances of convolutional layers, each of them consisting of 64 channels. The last layer is again a convolution with a $1\times1$ kernel size, and the final output is obtained by applying sigmoid activation function. In each denoising update step, we use the same strategy as proposed in~\cite{lequyer2022fast}, namely after each training epoch, the MSE between the noisy input image and ``denoised'' network output is computed. If this metric does not change much for 100 epochs, training is terminated, and the average of
the last one hundred validation images is the final denoised output.
%Using alternating minimisation, the joint optimisation scheme is presented in algorithm~\ref{algo_alternate}. 
A sketch of the alternating procedure is provided in Figure~\ref{fig:my_label}. 
%The segmentation algorithm~\ref{alg2} for the update of $u$ will be presented in the next Section. 
\begin{algorithm}
\caption{Alternating optimisation scheme.}\label{algo_alternate}
\begin{algorithmic}
\STATE{\textsc{Input}: noisy input image $f$ }
\STATE{\textsc{Initialisation}}: $u^0\leftarrow \boldsymbol{1}_{\{f>\epsilon\}}$, $\boldsymbol{W}^0=\boldsymbol{W}_0$, choose $\lambda>0$ and set $p\leftarrow 0.15$
\WHILE{$\mathcal{E}^k_{f,\lambda}(u^k,\boldsymbol{W}^k)/\mathcal{E}^{k-1}_{f,\lambda}(u^{k-1},\boldsymbol{W}^{k-1})\geq p\cdot\mathcal{E}^{k-1}_{f,\lambda}(u^{k-1},\boldsymbol{W}^{k-1})/\mathcal{E}^{k-2}_{f,\lambda}(u^{k-2},\boldsymbol{W}^{k-2})$}
\STATE{           $\boldsymbol{\boldsymbol{W}}^{k+1}\leftarrow\argmin_{\boldsymbol{W}} \mathcal{E}^k_{f,\lambda}(u^{k+1},\boldsymbol{W})$ 
            \COMMENT{with a few ADAM iterations for $F$ and Chan and Vese update for the background if $\Do_B(f)=\Do_B\mathbbm{1}$}}
            \STATE {$u^{k+1}\leftarrow\argmin_{u} \mathcal{E}^k_{f,\lambda}(u,\boldsymbol{W}^k)$  
            \COMMENT{with Algorithm~\ref{alg2})}}
\ENDWHILE
\end{algorithmic}
\end{algorithm}

In the following paragraph, we will discuss the convergence property of Algorithm~\ref{algo_alternate}.

\begin{figure*}[!ht]
    \centering
    \includegraphics[width = 0.99\columnwidth]{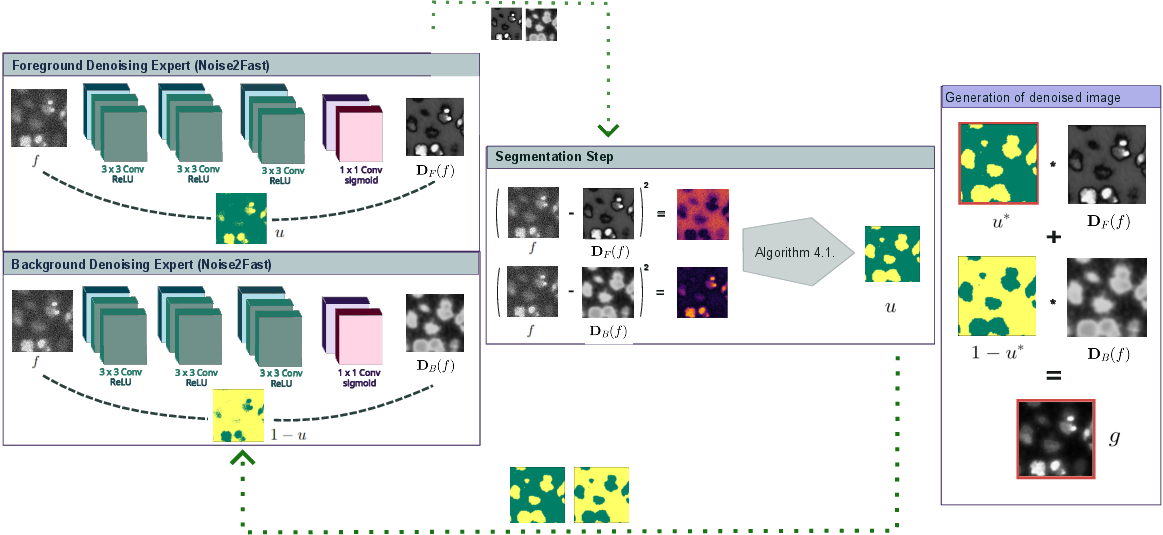}
    \caption{Alternating optimisation scheme. As a first step, regions are provided for the training of the two denoising experts using the  strategy. These regions can be obtained by thresholding image values or by manually choosing boxes. The differences between the given noisy image $f$ and network outputs $\Do_F(f)$ and $\Do_B(f)
$, are used in the subsequent segmentation step, minimising $\mathcal{E}_{\lambda,f}(\cdot, \boldsymbol{W})$ with Algorithm~\ref{alg2}. }
    \label{fig:my_label}
\end{figure*}

\subsection{Theoretical Results}\label{sec:theoretical}
In this section, we discuss some theoretical results of the proposed energy functional and the presented alternating algorithm. Note that these results hold if the denoiser is trained by minimising~\eqref{eq:denoisloss}.
%\begin{itemize}
%    \item[1] Alternate minimisation converges to a saddle point if the nonconvex  problem in $\theta$ admits a unique minimum ( that is reached by ADAM) for fixed $u$, see \cite{tseng2001convergence}[Theorem 4.1]. We will assume that the primal dual algorithm (resp. ADAM algorithm) actually minimises the energy in $u$ (resp. $\theta$).  
  %  \item [2] Thresholding Theorem: for fixed theta, the optimum in $u$ can be thresholded (rewrite the proof of Theorem 2 in \cite{chan2006algorithms}).
%\end{itemize}
\begin{remark}[Monotonicity of alternating minimisation]
    The proposed energy functional~\eqref{eq:joint2} is continuous and bounded from below.  
   Further, for each $k>0$, the following relations hold
\begin{align*}
    \mathcal{E}_{f,\lambda}(u^{(k)}, \boldsymbol{W}^{(k+1)})&\leq \mathcal{E}_{f,\lambda}(u^{(k)}, \boldsymbol{W}^{(k)})\\
        \mathcal{E}_{f,\lambda}(u^{(k+1)}, \boldsymbol{W}^{(k+1)})&\leq \mathcal{E}_{f,\lambda}(u^{(k)}, \boldsymbol{W}^{(k)}).
\end{align*}
Hence, the generated sequence $\{\mathcal{E}_{f,\lambda}(u^{(k)}, \boldsymbol{W}^{(k)})\}_{k\in\mathbb{N}}$ converges monotonically.
\end{remark}

%\begin{rem}[Monotonicity of alternating minimisation]  
%The proposed energy functional~\eqref{eq:joint2} is continuous and bounded from below. Therefore, for each $k\geq 0$, the following relations hold
%\begin{align*}
%    \mathcal{E}_{f,\lambda}(u^{(k)}, \boldsymbol{W}^{(k+1)})&\leq \mathcal{E}_{f,\lambda}(u^{(k-1)}, \boldsymbol{W}^{(k)})\\
 %       \mathcal{E}_{f,\lambda}(u^{(k+1)}, \boldsymbol{W}^{(k)})&\leq \mathcal{E}_{f,\lambda}(u^{(k)}, \boldsymbol{W}^{(k-1)}).
%\end{align*}
%Hence, the generated sequence $\{\mathcal{E}_{f,\lambda}(u^{(k)}, \boldsymbol{W}^{(k)})\}_{k\in\mathbb{N}}$ converges monotonically.
%\end{rem}

\begin{theorem}[Convergence of Algorithm~\ref{algo_alternate}]\label{convegence}
Assume that the level set $S^0=\{(u,\boldsymbol{W}): \mathcal {E}_{f,\lambda}(u,\boldsymbol{W})\leq\mathcal {E}_{f,\lambda}(u^0,\boldsymbol{W}^0)\}$ of $\mathcal {E}_{f,\lambda}$ defined in~\eqref{eq:joint2} is
compact and that $\mathcal {E}_{f,\lambda}$  is continuous on $S^0$. Then, the sequence $\{(u^k,\boldsymbol{W}^k\}$ generated by  Algorithm~\ref{algo_alternate} is defined and bounded. Moreover, %if the subproblem  $\min_{\boldsymbol{W}}\mathcal {E}_{f,\lambda}(u^k,\boldsymbol{W})$ is pseudoconvex or admits at most on minimiser, then 
every cluster point of $\{(u^k,\boldsymbol{W}^k)\}$ is a stationary
point of $\mathcal {E}_{f,\lambda}$.
\end{theorem}
\begin{proof}
 This is a direct application of Theorem 4.1 in~\cite{tseng2001convergence}, using that (i) we only alternate between two variables $u$ and $\boldsymbol{W}$, (ii) %the subproblem $\min_{u}\mathcal {E}_{f,\lambda}(u,\boldsymbol{W}^k)$ is convex in $u$ and (iii) 
 the coupling between $u$ and $\boldsymbol{W}$ in $\mathcal {E}_{f,\lambda}$ is smooth.
\end{proof}

\begin{remark}
    The energy~\eqref{eq:joint2}, which is convex for fixed network parameters $\boldsymbol{W}=(F,B)$ is a relaxation of the fully non-convex problem
    \begin{align}\label{non-convex}
        \mathcal{E}(\Sigma,\boldsymbol{W}) = \text{Per}(\Sigma,\Omega) + \int_{\Sigma}(f-\Do_F(f))^2 \dx  + \int_{\Omega\setminus\Sigma}(f-\Do_B(f))^2 \dx,
    \end{align}
    where $\Sigma\subset\mathbb{R}^2$, and $\Omega\setminus\Sigma$ are the two regions of the given image $f(x),$ and Per$(\Sigma,\Omega)$ is the perimeter of the interface separating these two regions.
\end{remark}
\begin{theorem}[Thresholding]\label{thresholding}
    For any fixed $\boldsymbol{W}$, a global minimiser for the non-convex problem $\min_{\Sigma,\boldsymbol{W}}\mathcal{E}(\cdot, \boldsymbol{W})$ in~\eqref{non-convex} can be found by carrying out the minimisation of $\min_u\mathcal{E}_{f,\lambda}(\cdot,\boldsymbol{W})$,
    and then setting $\Sigma(\tau)=\{x:u(x)\geq\tau\}$ for a.e. $\tau\in[0,1]$.
    \end{theorem}
\begin{proof}
    The proof is similar to the one in~\cite{chan2006algorithms}(Theorem 2). The only difference is in the data fidelity term, where instead of the fixed constants $c_1$, and $c_2$, we look at fixed network outputs $\Do_F(f)$, and $\Do_B(f)$. 
    By rewriting the $\lvert u\rvert_{\text{TV}}$-term with the co-area formula we show that
$\mathcal {E}_{f,\lambda}(u,\boldsymbol{W})=\int_{0}^{1} \mathcal {E}(\Sigma(\tau),\boldsymbol{W})-C \dtau$, with $C = \int_\Omega (f-\Do_F(f))^2\dx$ being independent of $u$. Thus, we conclude that if $u$ is a minimiser of the energy~\eqref{eq:joint2} for fixed $\boldsymbol{W}$, then for a.e. $\tau\in[0,1]$ the set $\Sigma(\tau)$ has to be a minimiser of~\eqref{non-convex}.
\end{proof}

\section{Numerical Implementation}\label{sec:numerical}
In the following, we describe the numerical implementation of the proposed method. 
\subsection{Segmentation Step}
We can rewrite our segmentation sub-problem in the form
\begin{align}\label{saddlepoint}
    \min_{u\in\mathbb{X}}\pdF(K(u)) + \mathcal{G}(u),
\end{align}
where $K(u) \coloneqq \nabla u$, $\mathcal{F}(v) \coloneqq \|v\|_{1,2},$ and $\mathcal{G}(u) \coloneqq i_\mathbb{A}(u)+\int_{\Omega}(f-\Do_F(f))^2 u + \int_{\Omega}(f-\Do_B(f))^2)(1-u)$.
It holds that $K:\mathbb{X}\rightarrow\mathbb{Y}$ is a linear mapping between Hilbert spaces $\mathbb{X},\mathbb{Y}$ and $\pdF:\mathbb{Y}\rightarrow [0,\infty]$ and $\mathcal{G}:\mathbb{X}\rightarrow [0,\infty]$ are convex and lower semi-continuous functionals, i.e. all the necessary assumptions for the application of the primal dual algorithm framework  proposed in~\cite{chambolle2011first} are fulfilled.

\subsubsection{Discretisation}
In the following, we fix the notation which we use throughout this Section. We  work with discrete images in $\Hr\coloneqq\mathbb{R}^{N_1\times N_2}$, denoting a finite dimensional Hilbert space equipped with an inner product $\langle u, v\rangle = \sum_i u[i]v[i]$ for $u,v\in\Hr$ with $i = (i_1,i_2)\in\{1,\dots,N_1\}\times\{1,\dots,N_2\}.$ The discrete gradient $\nabla = (\nabla_1,\nabla_2):\Hr\rightarrow\Hr\times\Hr$ is defined by forward differences with Neumann boundary conditions,
\begin{align*}
    (\grad_1 u)[i] &\coloneqq
    \begin{cases}
    	(u[i_1+1,i_2]-u[i_1,i_2]) / h & \text{if }  i_1<N_1\\
    	0 &  \text{if } i_1 = N_1
    \end{cases}
    \\[0.2em]
    (\grad_2 u)[i] &\coloneqq
    \begin{cases} (u[i_1,i_2+1]-u[i_1,i_2])/h & \text{if }  i_2 < N_2 \\
	0 &  \text{if }  i_2=N_2  \,.\end{cases}
\end{align*}
Its adjoint  is given by $\grad^* (v_1, v_2) = \grad^*_1 v_1 + \grad_2^* v_2 =: -\div (v_1, v_2)$  where $\div \colon \Hr \times \Hr \to \Hr$ is the  discrete  divergence operator and for  $(v_1, v_2) \in \Hr \times \Hr$ we have   
\begin{align*}
    (\grad^*_1 v_1)[i] & =
    \begin{cases}
    	-(v_1[i_1,i_2] - v_1[i_1-1,i_2]) / h & \text{if }  1<i_1<N_1\\
    	-v_1[1,i_2] &  \text{if } i_1 = 1 \\
    	\phantom{-}v_1[N_1-1,i_2] &  \text{if } i_1 = N_1 
    \end{cases}
    \\[0.1em]
    (\grad^*_2 v_2)[i] &=
    \begin{cases} 
    - (v_2[i_1,i_2] - v_2[i_1,i_2-1])/h & \text{if } 1< i_2 < N_2 \\
	 -v_2[i_1,1] &  \text{if } i_2 = 1 \\
    	\phantom{-}v_2[i_1,N_2-1] &  \text{if } i_2 = N_2  \,.
	\end{cases}
\end{align*}
The discrete, isotropic TV semi-norm of an image $u\in\Hr$ is defined as 
\begin{equation*}
    \norm{ \grad u}_{1,2} \coloneqq \sum_{i} \sqrt{(\grad_1 u[i])^2 + (\grad_2 u[i])^2} \,.
\end{equation*}
%The adjoint operator of the gradient is given by $\grad^{\ast}(v_1,v_2) = -\div(v_1,v_2)$ where $\div:\Hr\times\Hr\rightarrow\Hr$ is the discrete divergence operator.

The discrete versions of the admissible set and the corresponding indicator function, are $\U = \{u\in\Hr|0\leq u\leq 1\}$, and $i_{\U}.$ The discretisation of the data fidelity term of energy~\eqref{eq:joint2} is written as $\sum_{i}{\boldsymbol{F}}(u[i],\boldsymbol{W})$, where
\begin{align}
\boldsymbol{F}(u,\boldsymbol{W})&\coloneqq d(u,F)+d(1-u,B)\\ \nonumber
    d(u,F)&\coloneqq u\cdot\left(\Do_F(f)-f\right)^2 \\ \nonumber
    d(1-u, B)&\coloneqq (1-u)\cdot\left(\Do_B(f)-f\right)^2.
\end{align} 

Using these notations, the discrete version of energy~\eqref{eq:joint2} reads 
\begin{align}\label{energy:discrete}
    	\mathcal{E}_{f,\lambda}^{CV}(u,\boldsymbol{W}) 
	= i_{\U} (u) +  \lambda  \norm{\grad u}_{1,2} + \sum_i\boldsymbol{F}(u[i],\boldsymbol{W})\, .
\end{align}
The optimisation problem~\eqref{energy:discrete} is in general a non-convex and challenging problem to be solved. We will use alternating minimisation, where we employ the Chambolle-Pock algorithm~\cite{chambolle2011first} for the update step of the segmentation mask $u$, while for updating the network parameters $\boldsymbol{W}=(F,B)$, we apply ADAM optimisation~\cite{kingma2014adam}.

\subsubsection{Segmentation algorithm}
Here, we give details of the minimisation of of the functional~\eqref{eq:joint2} with respect to $u$ for fixed $\boldsymbol{W}$. This corresponds to solve problem~\eqref{saddlepoint} with
\begin{align*}
    \mathbb{X} = \Hr, \,
    \mathbb{Y} = \Hr^{2}, \,
    \pdF = \lambda\lVert v\rVert_{1,2}, \,
    K = \nabla,  \, 
    \mathcal{G} = i_{\mathbb{A}} +\sum_i\boldsymbol{F}(u[i],\boldsymbol{W}).
\end{align*}
As the operator $K$ is linear, and the functionals $\pdF$ and $\mathcal{G}$ are convex and lower semi-continuous, all requirements for the application of the primal dual algorithm proposed in~\cite{chambolle2011first} are fulfilled.  

To implement this algorithm, it is required to compute the Fenchel conjugate $\pdF^*$ of $\pdF$, as well as the proximal mappings of $\pdF^*$ and $\mathcal{G}$. 
We start with the derivation of the Fenchel conjugate of $\pdF$. For $\lVert\cdot\rVert_{1,2}$ it corresponds to the indicator function of the unit ball of the dual norm, resulting in $i_{2,\infty} = \lVert\cdot\rVert^{\ast}_{2,\infty}$. Hence we have $\pdF^*(v)= i_{2,\infty}(v/\lambda)$, the indicator function of $\{v|\lVert\boldsymbol{v}\rVert_{2,\infty}\leq 1\}\subset \Hr^2$. %For the second term of $\pdF$, the conjugate is given by the projection onto the all-ones matrix, i.e. the resulting Fenchel conjugate of the functional $\pdF$ reads
%\begin{align*}
%    \pdF^*(v,w) =   i_{2,\infty}(v/\lambda) + i_{\mathbbm{1}}(w) \,.
%\end{align*}
As a next step, we compute the proximal operators of $\mathcal{F}^{\ast}$ and $\mathcal{G}$. Recall that the proximal operator of the indicator function $i_{C}$ of some set $C$ is given by the orthogonal projection on $C$. The projection $P_{2,\infty}\colon\Hr\rightarrow\Hr^2$ onto the unit ball in the $(2,\infty)-$norm is thus obtained by
\begin{align*}
    (P_{2, \infty} (\vv))[i,k]= \frac{v[i,k]}{\max\{1, (v_1[i,k]^2 + v_2[i,k]^2)^{1/2}\}}.
\end{align*}
%For the second indicator function, $i_{\mathbbm{1}}$, it holds that $P_{\mathbbm{1}}(w)=\mathbbm{1}$.
Thus, the proximal operator of $\mathcal{F}^{\ast}$ results in
\begin{align*}
    \prox_{\mathcal{F}^{\ast}}(\vv) = P_{2, \infty, \lambda}(\vv) \coloneqq P_{2, \infty} (\vv/\lambda).
\end{align*}
%where $P_{2,\infty,\lambda}(\vv)\coloneqq P_{2, \infty} (\vv/\lambda)$.
Further, by introducing $\tilde f=\left(f-\Do_F(f)\right)^2-\left(f-\Do_B(f)\right)^2$, one can finally show that
%$$h(u)=i_{\U} (u) + \int_\Omega \left(f(x)-\Do_F(f)(x)\right)^2 u(x) +\int_\Omega \left(f(x)-\Do_B(f)(x)\right)^2(1-u(x))\,,$$
%so that
%\begin{equation}\label{dual}
%\mathcal{E}_{f,\lambda}(u,\theta) 
%=  \max_{v}\langle u,v\rangle -  i_{2,\infty,\lambda}(v) + h(u)\,.
%\end{equation}
%Denoting as $g=\left(f-\Do_F(f)\right)^2-\left(f-\Do_B(f)\right)^2$, we can show that:
$$\prox_{\tau \mathcal{G}}(u_0[i])=P_{\U}\left(u_0[i]-\tau \tilde f[i]\right).$$
The overall  primal dual Algorithm~\ref{alg2} is summarised below.
\begin{algorithm}
\caption{Segmentation algorithm based on the minimisation of the energy functional~\eqref{eq:joint2} with respect to $u$ for a fixed $\boldsymbol{W}$.}\label{alg2}
\begin{algorithmic}
\STATE{\textsc{Input}: noisy input image $f \in \Hr$ }
\STATE{\textsc{initialisation}:  $v^0\in  \Hr$,  $u^0, \bar u^0 \in \Hr$}
\WHILE{$\lVert u^{n+1}-u^n\rVert>\epsilon$}
\STATE{$v^{n+1} \gets P_{2, \infty,\lambda} ( v^{n} +  \sigma \boldsymbol{\nabla} \bar u^n  )$}
\STATE{$u^{n+1} \gets P_{\U}(u^{n}  -  \tau  \boldsymbol{\nabla}^\intercal v^{n+1}-\tau \tilde f)$}
\STATE{$\bar u^{n+1} \gets u^{n+1} + \eta(u^{n+1}-u^{n})$.}
\ENDWHILE
\RETURN $u^{n+1}$
\end{algorithmic}
\end{algorithm}

\subsection{Acceleration with a mask prior}\label{sec:mask}
In our experiments, we observed, that one of the denoisers (the one that is trained on the more complex region), tends to improve on the other region. Once this progress has started, it is quite difficult to stop the segmentation mask from expanding and converging to an undesired minimum being the constant segmentation result.
Inspired by the work in%of Baeza et al.~\cite{baeza2010narrow}, and Yildizoglu et al.
~\cite{baeza2010narrow,yildizoglu2012active}, we now propose to overcome the problem of finding an appropriate stopping criteria, by adding a fidelity term, ensuring that the updated segmentation mask $u^k$ does not deviate too far from its initial guess. 
Assume that we have a reference mask $u_R^0$, then, in the continuous setting, we consider the successive problems:
\begin{equation}\label{eq:joint2_ref}
\begin{aligned}\mathcal{E}^k_{f,\la}(u,\Do_F,\Do_B) 
\coloneqq &i_\U (u) + \lambda \lvert u\rvert_{\text{TV}}+  \int_{\Omega} \left(f-\Do_F(f)\right)^2 u\dx \\&+\int_{\Omega}\left(f-\Do_B(f)\right)^2(1-u)\dx  + \frac\mu{2}||u-u_R^k||^2\,.\end{aligned}
\end{equation}

We can therefore optimise iteratively problem~\eqref{eq:joint2_ref} with the alternate procedure presented in Algorithm~\ref{algo_alternate_acc}. Note, that in this case, as the global energy is changed at each iteration, we do not have convergence guarantee anymore for the alternating procedure.
\begin{algorithm}
\caption{Alternating optimisation scheme with acceleration.}\label{algo_alternate_acc}
\begin{algorithmic}
\STATE{\textsc{Input}: noisy input image $f$ }
\STATE \textsc{Initialisation:} $u^0\leftarrow f$ and $\boldsymbol{W}^0=\boldsymbol{W}_0$, $u_R^0$, choose $\lambda>0$  
\FOR{$k = 1,\dots, N$}
\STATE{           $\boldsymbol{\boldsymbol{W}}^{k+1}\leftarrow\argmin_{\boldsymbol{W}} \mathcal{E}^k_{f,\lambda}(u^{k+1},\boldsymbol{W})$ 
            \COMMENT{with a few ADAM iterations for $F$ and Chan and Vese update for the background if $\Do_B(f)=\Do_B\cdot\mathbbm{1}$}}
            \STATE { $u^{k+1}\leftarrow\argmin_{u} \mathcal{E}^k_{f,\la}(\uu,\boldsymbol{W}^k)$  
            \COMMENT{with Algorithm~\ref{alg_acc})}}
\STATE{$\uu_R^{k+1}=\uu^{k+1}$ (update reference mask)}
\ENDFOR
\end{algorithmic}
\end{algorithm}

%\textbf{Initialise:} $u^0\leftarrow f$ and $\theta^0=\theta_0$ and $u_R^0$
%\begin{enumerate}
 %   \item $u^{k+1}\leftarrow\argmin_{u} \mathcal{E}^k_{f,\la}(\uu,\theta^k)$  (with algorithm~\ref{alg3})
  %  \item $\theta^{k+1}\leftarrow\argmin_\theta \mathcal{E}^k_{f,\la}(u^{k+1},\theta)$  (with a few ADAM iterations for $F$ and Chan and Vese update for the Background if $\Do_B(f)=\Do_B\mathbbm{1}$) 
   % \item $\uu_R^{k+1}=\uu^{k+1}$ (update reference mask)
  %  \end{enumerate}
To solve the segmentation problem, we reformulate the optimisation of problem~\eqref{eq:joint2_ref} for fixed $\boldsymbol{W}$ as $\min_u\pdF(K(u)) +\mathcal{G}^k(u)$, with 
\begin{align*}\mathcal{G}^k(u)=i_\U (u) &+\frac\mu{2}||u-u_R^k||^2  +  \int_\Omega \left(f(x)-\Do_F(f)(x)\right)^2 u(x)\dx \\&+\int_\Omega \left(f(x)-\Do_B(f)(x)\right)^2(1-u(x))\dx.\end{align*}
%so that
%\begin{equation} \label{eq:joint4bis}\mathcal{E}^k_{f,\la}(u,\theta) 
%\coloneqq  \max_{v}\langle u,v\rangle -  i_{2,\infty,\lambda}(v) + h(u^k)\,.
%\end{equation}
Recalling that $\tilde f=\left(f(x)-\Do_F(f)(x)\right)^2-\left(f-\Do_B(f)(x)\right)^2$, we can show that:
$$\prox_{\tau \mathcal{G}^k}(u^0[i])=P_\U\left(\frac{u^0[i]+\tau\mu u_R^k[i]-\tau \tilde f[i]}{1+\tau\mu}\right).$$
Observing that $\mathcal{G}^k$ is $\mu$-strongly convex  in $u$, we consider the accelerated primal dual algorithm of~\cite{chambolle2011first} to solve problem~\eqref{eq:joint2_ref}.

\begin{algorithm} \label{alg_acc}
\caption{Segmentation algorithm based on the minimisation of the energy functional~\eqref{eq:joint2_ref} with respect to $u$ for a fixed $\boldsymbol{W}$.}
\begin{algorithmic}
\STATE{\textsc{Input}: noisy input image $f \in \Hr$ }
\STATE{\textsc{Initialisation}:  $v^0\in  \Hr$,  $u^0, \bar u^0 \in \Hr, \lambda>0, \sigma>0, \tau>0, \boldsymbol{W}$ fixed}
\WHILE{$\lVert u^{n+1}-u^n\rVert>\epsilon$}
\STATE{$v^{n+1} \gets P_{2, \infty,\lambda} ( v^{n} +  \sigma \boldsymbol{\nabla} \bar u^n  )$}
\STATE{$u^{n+1} \gets P_{\U}\left((u^{n}  -  \tau  \boldsymbol{\nabla}^\intercal v^{n+1}+\tau\mu u_R^k-\tau \tilde f )/(1+\tau\mu )\right)$}
\STATE{$\eta=\frac{1}{1+2\mu\tau}, \tau=\tau\eta, \sigma = \frac{\sigma}{\eta}$}
\STATE{$\bar u^{n+1} \gets u^{n+1} + \eta(u^{n+1}-u^{n})$.}
\ENDWHILE
\RETURN $u^{n+1}$
\end{algorithmic}
\end{algorithm}
As we have discussed the numerical implementation of the segmentation step, we now present the discrete setting and implementation of the denoising step.
\subsection{Denoising step using Noise2Fast strategy}
We here detail the denoising of a discretized 2D image $f\in\mathbb{R}^{m\times n}$ composed of a clean signal $g\in\mathbb{R}^{n\times n}$ and noise $n\in\mathbb{R}^{m\times n}$, i.e.
\begin{align*}
    f = g + n.
\end{align*}
For completeness, we introduce $u_B^k,$ which for $k=0$ corresponds to the initialisation of the background region. 
These masks can either be obtained by thresholding the image, or can be given in form of user-provided boxes. For the next update steps, i.e. $k=1,\dots, N$ it holds that $u_B^k=1-u^k.$
Using these notations, for fixed $\boldsymbol{u}^k = (u^k, u_B^k)$, in the $k$-th denoising step of our alternating procedure, the energy functional \eqref{eq:joint2}, reduces to 
\begin{align}\label{denoising_min}
    \min_{\boldsymbol{W}}\sum_i{\boldsymbol{F}}(\boldsymbol{u}^k[i],\boldsymbol{W})=    \min_{F,B}\sum_i \left(\Do_F(f)[i]-f[i]\right)^2\cdot u^k[i]    +\left(\Do_B(f)[i]-f[i]\right)^2\cdot u_B^k[i],
\end{align}
where $\Do_F(f)$ and $\Do_B(f)$ are again the (deep) experts  respectively dedicated to the denoising of the foreground and background. 

We build our denoisers on top of the Noise2Fast method introduced by Lequyer et al. in~\cite{lequyer2022fast}. The authors propose a fast single image blind denoiser, using a special downsampling strategy. More precisely, their method consists in splitting a given image into smaller parts by using a checkerboard downsampling strategy. From a single image, four images are thus generated, by removing one half of all pixels, and shifting the remaining pixels to fill in the gaps left behind. Then, a network is trained to learn the mappings between the resulting downsampled image pairs. Due to the internal redundancy in form of recurrent patches present in images, and the high degree of self-similarity, the neural network will also be able to denoise the whole image instead of the downsampled ones~\cite{buades2011self,zontak2011internal,glasner2009super}. For a more detailed description of the Noise2Fast training strategy, such as the network architecture, we refer the reader to~\cite{lequyer2022fast}.  
%Then, for $(x,y)\in\mathcal{T}$, a neural network $\Phi_{\theta}$ is trained to learn the mapping $\Phi_{\theta}(x)\rightarrow y.$

%The motivation behind that strategy lies in the observation that 
%\begin{align*}
 %   \Phi_{\theta}(x)\rightarrow g_{\text{even}} + n_{\text{odd}} + (g_{\text{odd}}-g_{\text{even}}),
%\end{align*}
%together with the fact that for every $(i,j)\in\mathbb{N}_{\leq m} \times \mathbb{N}_{\leq n}$, it holds that $g_{\text{odd}}(i,j)$ and $g_{\text{even}}(i,j)$ are adjacent pixels in the original image signal, causing this term to be small almost everywhere.

In our approach, we use a different loss function as the one described in the work of Lequyer et al~\cite{lequyer2022fast}. Instead of considering the whole image domain for training,  we restrict the optimisation process for the foreground $\Do_F(f)$ (resp. background $\Do_B(f)$) expert to the current segmentation masks $u^k$ (resp. $1-u^k$) obtained by Algorithm~\ref{alg2}.

In a first step, as in~\cite{lequyer2022fast} the downsampled training images are generated in the following way
\begin{align*}
    &f_{\text{even}}(i,j) = f\left(i, 2j + (i\,\text{mod}\, 2)\right)\in\mathbb{R}^{m\times\frac{n}{2}}\\
       & f_{\text{odd}}(i,j) = f\left(i, 2j + (i\,\text{mod}\, 2) + 1)\right)\in\mathbb{R}^{m\times\frac{n}{2}}\\         &f^{\prime}_{\text{even}}(i,j) = f\left(2i + (i\,\text{mod}\, 2),j\right)\in\mathbb{R}^{\frac{m}{2}\times n}\\
       & f^{\prime}_{\text{odd}}(i,j) = f\left( 2i + (i\,\text{mod}\, 2) + 1), j\right)\in\mathbb{R}^{\frac{m}{2}\times n},
\end{align*}
and we repeat this downsampling procedure for the segmentation masks $u^k$ and $u_B^k$, for $k=0,\dots, N$ as well. 
We denote as \begin{align*}
 \mathcal{J}^k=\{(f_{\text{even}},f_{\text{odd}}, u_{\text{odd}}^k,u_{B\text{,odd}}^k), (f_{\text{odd}},f_{\text{even}},u_{\text{even}}^k,u_{B\text{,even}}^k),  \\(f^{\prime}_{\text{even}},f^{\prime}_{\text{odd}},u_{\text{odd}}^{k^\prime},u_{B\text{,odd}}^{k^\prime}), (f^{\prime}_{\text{odd}},f^{\prime}_{\text{even}},u_{\text{even}}^{k^\prime},u_{B\text{,even}}^{k^\prime})\}
\end{align*}
%\mathcal{J}^k=\{(f_{\text{even}},f_{\text{odd}}, u_{F\text{,odd}}^k,u_{B\text{,odd}}^k), (f_{\text{odd}},f_{\text{even}},u_{F\text{,even}}^k,u_{B\text{,even}}^k), (f^{\prime}_{\text{even}},f^{\prime}_{\text{odd}},u_{F\text{,odd}}^{k,\prime},,u_{B\text{,odd}}^{k,\prime}), (f^{\prime}_{\text{odd}},f^{\prime}_{\text{even}})\} the set of training pairs.
the training set for $k=0,\dots,N$, with $N$ being the number of iterations of the alternating minimisation.

We then train the two denoising networks, $\Do_F(f)$ and $\Do_B(f)$, restricted to the given regions, $u^k$, and $u^k_B$, i.e. for $(\tilde{f},\tilde{g}, \tilde{u}, \tilde{u}_B)\in\mathcal{J}^k$ we minimise
%\begin{align*}
 %   &\Do_F(f)(f_{\text{even}})\lvert_{u^0_F}\rightarrow f_{\text{odd}}\lvert_{u_F^0},\,\, \Do_F(f)(f_{\text{odd}})\lvert_{u^0_F}\rightarrow f_{\text{odd}}\lvert_{u_F^0} \\
  %    &\Phi_{\Do_B}(f_{\text{even}})\lvert_{u^0_B}\rightarrow f_{\text{even}}\lvert_{u_B^0}\,\, \Phi_{\Do_B}(f_{\text{odd}})\lvert_{u^0_B}\rightarrow f_{\text{even}}\lvert_{u_B^0} 
%\end{align*}
\begin{align}\label{eq:noise2fast}\mathcal{L}_{\boldsymbol{u}}^k(\boldsymbol{W})=\sum_i \left(\Do_F(\tilde{f})[i]-\tilde{g}[i])\right)^2\cdot \tilde{u}[i] + \left(\Do_B(\tilde{f})[i]-\tilde{g}([i])\right)^2\cdot \tilde{u}_B[i].
\end{align}
Thus the self-supervised denoisers learn to reconstruct even (resp. odd) pixels and $\tilde f$ of the image thanks to the odd (resp. even) ones $\tilde g$. 
As mentioned above, caused by the self-similarity redundancy, by minimising~\eqref{eq:noise2fast}, $\mathcal{L}_u^k$, we also solve problem~\eqref{denoising_min}.

In the next paragraph, we demonstrate the possible applications of three different variants of the proposed joint denoising and segmentation method.
The code is provided on GitHub (\href{https://github.com/Nadja1611/Self2Seg.git}{https://github.com/Nadja1611/Self2Seg.git}).
%\begin{align*}
 %   f_{\text{even}}(i,j) &= f(i,2j + i \text{mod}\, 2)),\\
 %   f_{\text{odd}}(i,j) &= f(i, 2j + (i\,\text{mod}\,2) +1).
%\end{align*}
%Then, the idea is to train a neural network to learn the mapping
%\begin{align*}
 %   f_{\theta}(x_{\text{even}})\rightarrow x_{\text{odd}},
%\end{align*}
%which can be expressed as 
%\begin{align*}
%    f_{\theta}(s_{\text{even}}+n_{\text{even}}) \rightarrow s_{\text{even}} + n_{\text{odd}} + (s_{\text{odd}}- s_{\text{even}}).
%\end{align*}

\section{Experiments and Results}\label{sec:experiments}

As a first application, we test our method on the microscopy cell nuclei dataset from the DSB2018 dataset\footnote{https://www.kaggle.com/c/data-science-bowl-2018} stemming from the Kaggle 2018 Data Science Bowl challenge. 
The data consists of a diverse collection of cell nuclei imaged by various fluorescence microscopes. The patches are of size $128 \times 128$, and come with manually generated segmentation ground truths.
More precisely, we use the noise free data and manually add gaussian noise with three different noise levels, namely 10, 30, and 50. In our experiments, we considered the same subset of images as the one used in~\cite{buchholz2021denoiseg}, where the authors demonstrated that the segmentation of noisy data can be improved by addressing denoising and segmentation in a cooperative (but not fully joint) manner.

In the following experiments, for the evaluation of the segmentation performance we use the Dice metric, and for capturing the denoising performance in the experiments, we choose peak signal to noise ratio (PSNR) and structural similarity metric (SSIM).

%For our experiments, we use early stopping. %, as otherwise, if the segmentation mask for one region contains pixels of the other region, together with the tendency of the total variation, to produce a constant segmentation result, over time, the performance difference between the two denoisers in the different regions decreases such that the final result obtained is a constant image. 
We stop our alternating Algorithm~\ref{algo_alternate}, as soon as the decrease of energy~\eqref{eq:joint2} is less than 15 percent of the previous decrease rate. We tried out a few different strategies, and this one turned out to be the most promising one. We indeed observed that a criteria based on the change in the energy decay is robust to different scales for the regularisation parameter $\lambda$, and it also adapts to different type of images.

We compare the segmentation performance of our joint denoising and segmentation approach with the convex Chan-Vese model from~\cite{chan2006algorithms} applied either on the noisy images directly, or on the previously denoised data within a sequential approach. 
For both the proposed joint approach and the sequential one, we use the same denoised image as starting point for fair comparisons.
Further, we test our method against the two-stage method~\cite{cai2013two}, and the partially joint deep learning based denoising and segmentation framework of~\cite{buchholz2021denoiseg}.

\subsection*{Segmentation with the constant background assumption}
We start with the evaluation of our method on a subset of the DSB2018 cell nuclei data which were manually corrupted (noise levels 10, 30 and 50). To this end, we train a foreground denoiser, $\Do_F(f)$, and we assume the background to be constant, i.e. $\Do_B=B\odot\mathbbm{1}.$ For this particular type of images, this assumption is useful, while for images with more structural patterns, this may not be a reasonable choice, and two denoising experts might be necessary. 

To apply our joint approach, we first denoise the given image using the Noise2Fast strategy in the way as described in Section~\ref{sec:numerical}, and use the thresholded denoised image (with the threshold $\epsilon$ set to $0.5$) as initialisation. For noise level 10, we applied the segmentation Algorithm~\ref{algo_alternate} with the constant background assumption, while a Noise2Fast expert was considered for higher noise levels. We recall that the overall process for solving the joint segmentation and denoising is presented in Algorithm~\ref{algo_alternate}. Depending on the type of image, for the alternate process between two and six iterations are required to meet the convergence criteria.

For each method, we conducted the experiments with ten different values of the regularisation parameter $\lambda$ evenly distributed in the interval $[0,1]$, and then selected for each image the result with the highest Dice value.

As a further comparison, we applied the convex Chan-Vese model from~\cite{chan2006algorithms} directly on the noisy images. The obtained results are depicted in Figures~\ref{cell_n10} to~\ref{cell_n50}, while the segmentation evaluation metrics are summarised in Table~\ref{tab_dice}. We observe that for all three noise levels, the sequential and Chan-Vese method from~\cite{chan2006algorithms} struggle with intensity inhomogenities of the cells. These examples highlight the strength of the proposed unified approach, which is capable of segmenting cells with intensities close to the mean value of the background. Notice that the proposed approach does not perform well on the last example due to the presence of intensity inhomogeneities, ascribed to a spatially varying field, the bias field, in the upper left corner of the image. Please note that in this case, evaluating the denoising performance might not be appropriate, as we are assuming a constant background and we do not apply denoising to the background.

Further, we compared to the two-step approach outlined by Cai et al. in \cite{cai2013two}. We employed the provided MATLAB code and fine-tuned the parameter $\lambda$ for three distinct noise levels: $n10$ with $\lambda=4.5$, $n30$ with $\lambda=3.5$, and $n50$ with $\lambda=3.2$. Upon analysis of the outcomes, it became evident that this method encounters difficulties with intensity inhomogeneities and tends to yield excessively smoothed results.

In Table~\ref{tab_dice}, the results obtained by the supervised joint denoising and segmentation method \textsc{DenoiSeg}~\cite{buchholz2021denoiseg}, are summarised. Here, we ran the provided code randomly using 10 training images. More precisely, we used the DSB2018 dataset with noise level zero and added Gaussian noise in the same way as before to all of the 4320 images, among which 3750 were used for training, 670 for validation and the same 50 as we used for our experiments for testing. It has to be mentioned that for validation, all 570 annotated validation images are used in~\cite{buchholz2021denoiseg}, resulting in a total number of 580 annotated images in total during the training process. As displayed in Table~\ref{tab_dice}, this method performs the best. To have a fairer comparison in terms of training data, we decided to adapt their method by using 10 images in total (\textsc{DenoiSeg} 10 in Table~\ref{tab_dice}), 7 for training and 3 for validation. In this setting,  all available data are still used for the training of the denoiser, whereas for the segmentation network, the ground truth masks for all but the ten training images are zeroed out. With this smaller level of supervision, our approach outperforms the method of~\cite{buchholz2021denoiseg}.

\begin{figure}[ht]
\centering
    \includegraphics[width=0.93\columnwidth]{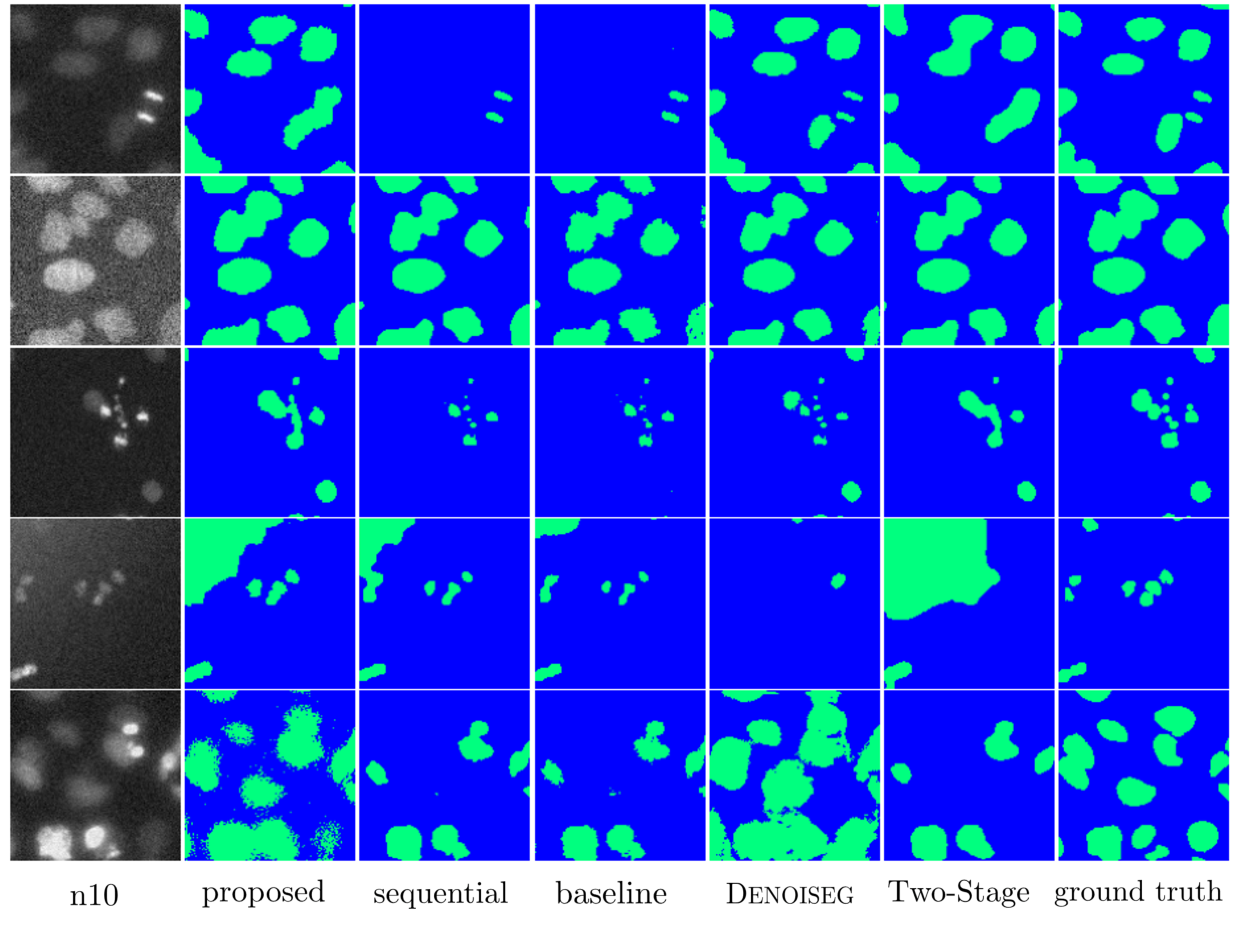}
    \caption{\textbf{Visual comparison of the segmentation results of data with noise level 10.} From left to right, this figure shows: the noisy input, the results obtained with the proposed joint approach, the sequential approach, the chan-Vese baseline and the ground truth segmentation masks. For all compared methods, the $\lambda$  maximising the Dice score has been  selected.}\label{cell_n10}\end{figure}

\begin{figure}[ht]
\centering
    \includegraphics[width=0.93\columnwidth]{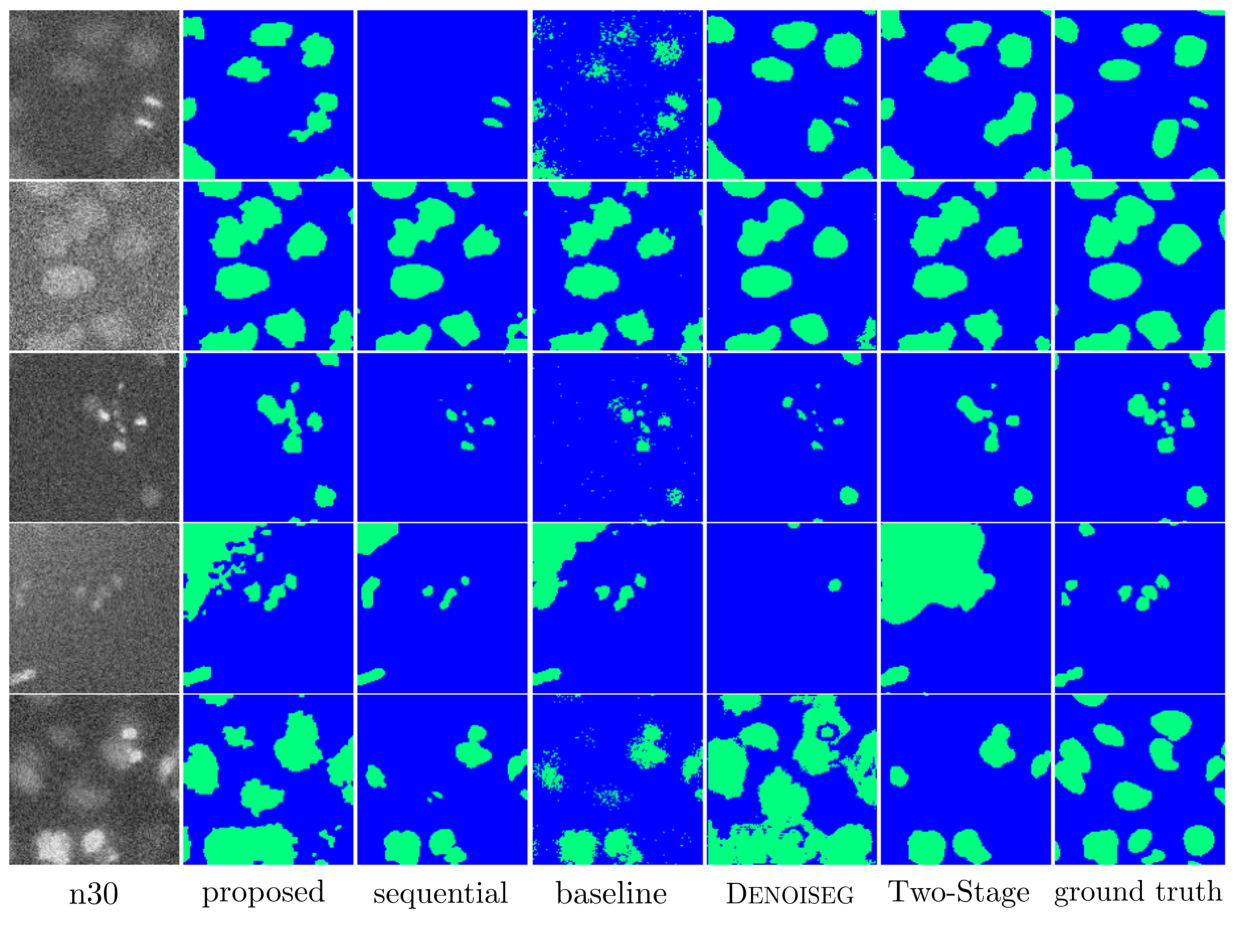}
    \caption{\textbf{Visual comparison of the segmentation results of data with noise level 30.}}\label{cell_n30}\end{figure}
    
For higher noise levels, it is required to filter the background fidelity term. This prevents from considering higher values of the regularisation parameter $\lambda$, that may lead to an over-segmentation of the background and an overall decrease of the segmentation performance. For noise level 30 and 50, as mentioned in Section~\ref{sec:proposed}, we therefore minimise
\begin{equation*} \label{eq:joint_cells}
\begin{split}
\mathcal {E}_{f,\lambda}(u,\boldsymbol{W}) 
= i_{\U}(u) &+ \lambda\lvert u\rvert_{\text{TV}}\dx + \int_\Omega \left[K_{\sigma}\ast\left(f-\Do_F(f)\right)\right]^2 u(x)\dx\\ 
&+\int_{\Omega}\left[K_{\sigma}\ast\left(f-\Do_B(f)\right)\right] ^2(1-u(x))\dx\, 
\end{split}
\end{equation*}
with $K_{\sigma}$ being a mean filter with $\sigma=3$.

The next paragraph shows experimental results which were obtained applying our idea of training denoising experts for both regions. 
\begin{figure}[ht]
\centering
    \includegraphics[width=0.93\columnwidth]{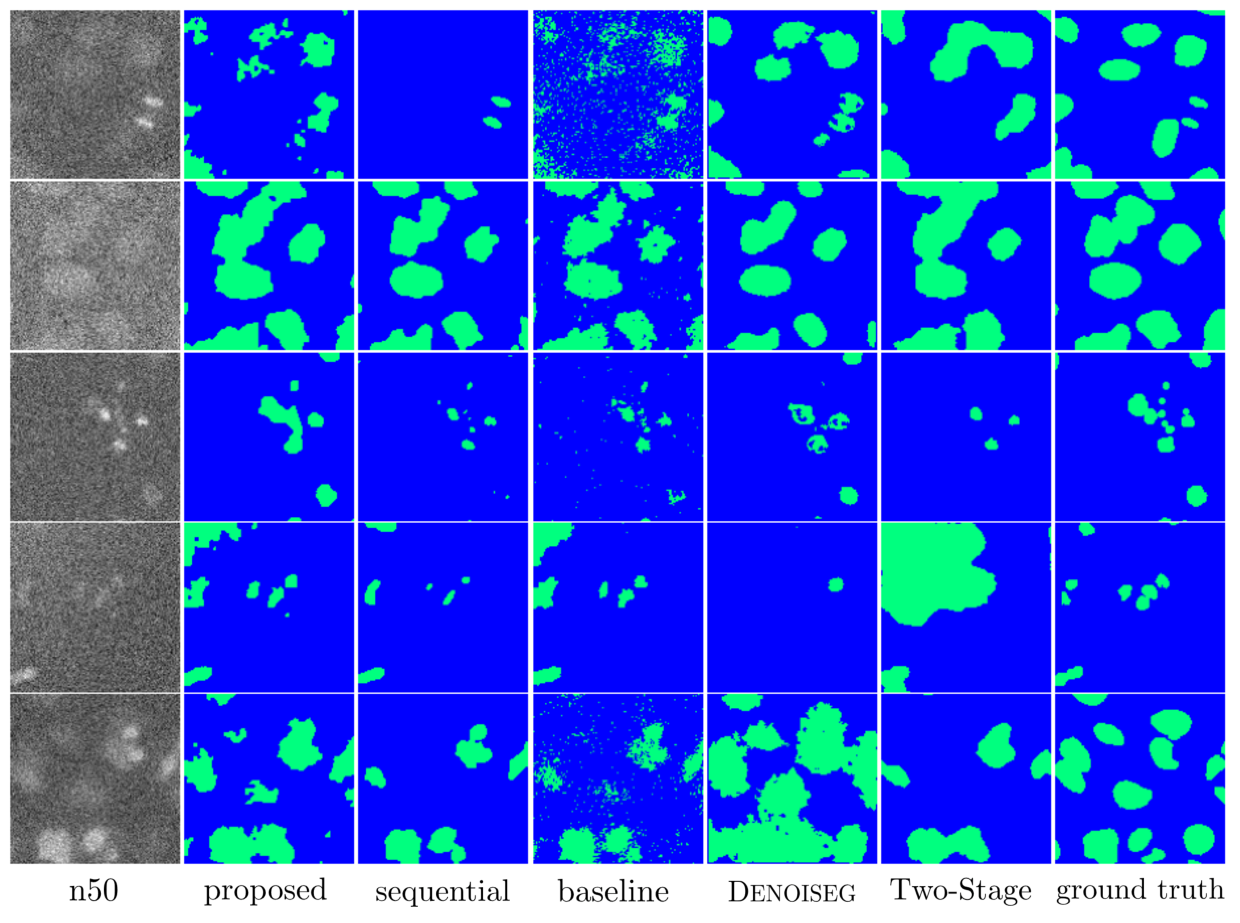}
    \caption{\textbf{Visual comparison of the segmentation results of data with noise level 50.} }\label{cell_n50}\end{figure}
\begin{table}[h]
\footnotesize
\begin{center}
\begin{tabular}{c|c|c|c|c|c|c}
\hline
\textbf{Noise} & \textbf{Baseline} & \textbf{Sequential} & \textbf{Proposed} & \textbf{TwoStage~\cite{cai2013two}} & \textbf{\textsc{DenoiSeg}~\cite{buchholz2021denoiseg}} & \textbf{\textsc{DenoiSeg} 10} \\
\hline
$n10$ & 0.820 & 0.799 & \underline{0.851} & 0.829 & \textbf{0.864} & 0.843 \\
$n30$ & 0.773 & 0.777 & \underline{0.825} & 0.800 & \textbf{0.848} & 0.820 \\
$n50$ & 0.582 & 0.735 & \underline{0.786} & 0.758 & \textbf{0.818} & 0.750 \\
\hline
\end{tabular}
\end{center}
\caption{\label{tab_dice}Dice values obtained on 50 images of the DSB2018 dataset for the compared methods, and three different noise levels. Here, baseline is the convex Chan-Vese~\cite{chan2006algorithms} method directly applied to the noisy data, while for the sequential method, we first denoise the image using Noise2Fast~\cite{lequyer2022fast}.}
\end{table}

\subsection*{Segmentation using two denoisers}
In the toy example in Figure~\ref{fig:toy} from Section~\ref{sec:contributions}, we trained two denoising experts (in this case we used a linear network consisting of one filter of size $15\times 15$) initialised by the yellow and purple boxes of size $30\times 30$. We iterated between the denoising and segmentation three times, until the energy decrease was less then 10 percent. For segmentation, we set the regularisation parameter $\lambda$ to 0.02. After the first segmentation step, the loss functions of the denoisers were restricted to $u$ and $1-u$ respectively. %Three alternations between denoising and segmentation were required to obtain the final result. 

Figure~\ref{brodatz} is a typical example showing the strength of the proposed algorithm compared to intensity-based approaches. In this experiment, we preprocessed the given image of size 256$\times$256 in a way, that both regions have the same mean value, and added gaussian noise as described before, with a noise level of 10. As a  consequence, the classical Chan-Vese algorithm totally fails on this example. This model can nevertheless perform well with an adapted hand-crafted transformation of the image to segment. As illustrated in the two last images of Figure~\ref{brodatz}, when fed with the map of normalized image gradient instead of the original image intensities, the Chan-Vese model is able to segment the two part of the image. 

On the other hand, our approach is able to automatically learn a relevant transformation of image data and provides excellent segmentation without any previous trick. The reason for that is again, that the weights  learnt by the two denoising experts strongly depend on the true underlying signal, which, in contrast to the mean intensity, is different in the two regions. Here, both denoising experts were initialised by boxes of size 50$\times$50 centered in the regions. We used a regularisation parameter $\lambda$ of 0.06, and set the learning rate to 0.001. Using the same stopping criterion as in the cell example, these results were obtained after $3$ iterations of the alternating procedure involving denoising and segmentation steps.

\begin{figure}[ht]
    \centering
    \includegraphics[width = 0.93\columnwidth]{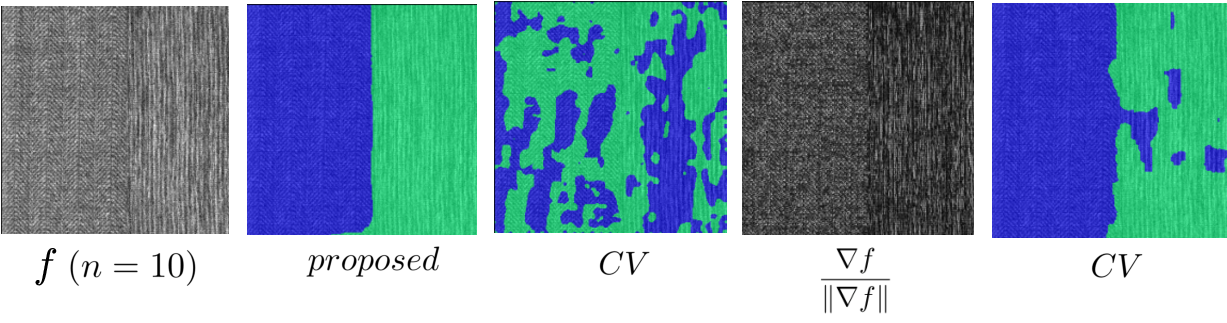}
    \caption{\textbf{Segmentation of a noisy Brodatz image consisting of two different textures.} The first three images show the noisy input $f$, the minimiser of energy~\eqref{eq:joint2}, and the result obtained by directly applying the active contour algorithm~\cite{chan1999active}. The fourth image shows the normalized gradient of $f$, and the last one is the result obtained when applying the classical Chan-Vese algorithm on the normalized gradient map.}
    \label{brodatz}
\end{figure}
In Figure~\ref{denoising_brod}, we display the clean image considered in the experiment of Figure~\ref{brodatz}, as well as different denoised images with their corresponding quantitative metrics. More precisely, the second image in the figure is obtained by applying the Noise2Fast strategy to the whole image, while the third image is the result of the proposed joint optimisation procedure, where the image is composed using the segmentation mask $u$ and the denoised images from the two denoising experts. Especially in the left regions, we can observe a  better denoising performance of the proposed method, which is more evident by examining the PSNR (20.36 vs 19.815) and SSIM (0.753 vs 0.696) values.
\begin{figure}[ht]
    \centering
    \includegraphics[width = 0.93\columnwidth]{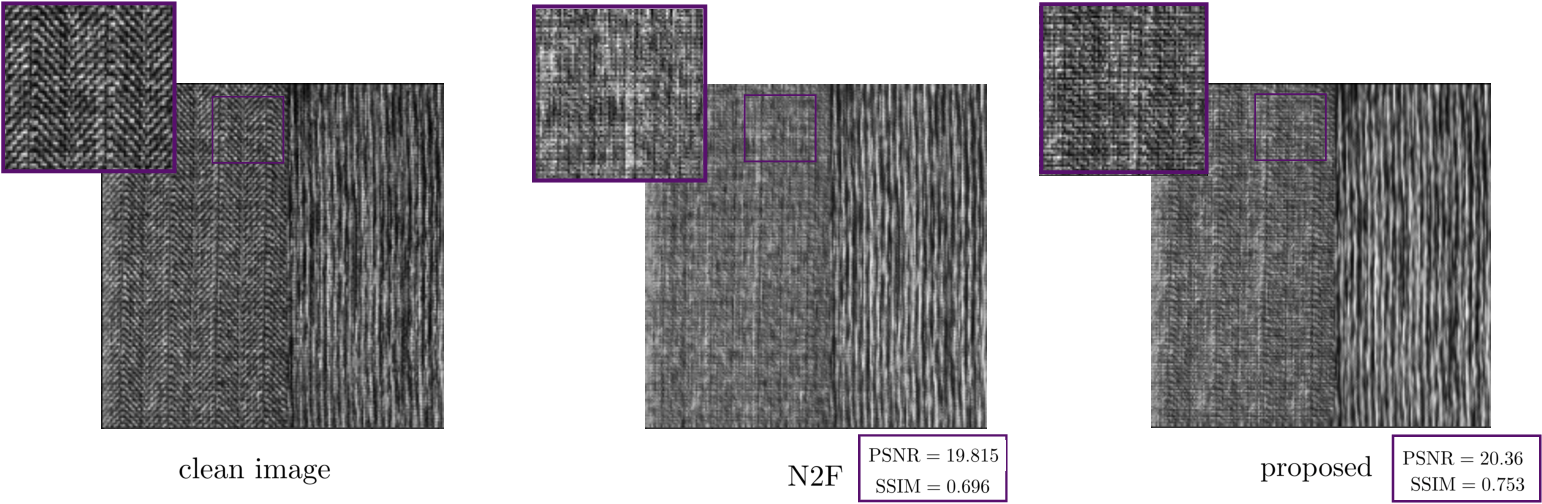}
    \caption{\textbf{Comparison of denoising performance with different Noise2Fast strategies}. On the middle image, Noise2Fast is applied to the whole image. On the right image, we present the final denoised image obtained from the two separate denoisers learned with the proposed framework.}
    \label{denoising_brod}
\end{figure}
\subsection*{Segmentation with a reference mask using Algorithm~\ref{algo_alternate_acc}}
In Figure~\ref{squirrel}, we show another example of image segmentation for three different noise levels using Algorithm~\ref{alg_acc}. The main difficulty of this image lies in the intensities which are shared by the object to be segmented and the background. Therefore, we chose a representative box for initialising the squirrel, which includes both, dark and bright areas, in order to enable the foreground denoising expert to better generalize on the foreground region consisting of dark and bright areas. Naturally, as the squirrel and background do not differ much in terms of their structural properties, the foreground denoiser, $\Do_F(f)$ also performs well on the background, causing the segmentation mask $u$ to grow. In order to control this behaviour, we applied our second strategy that includes a recursive reference mask as described in Algorithm~\ref{algo_alternate_acc}, thus preventing the segmentation mask obtained at iteration $k+1$ from deviating too much from the previous one at iteration $k$. More precisely, the parameters that we used for noise level 10 were $\mu=0.0001,\lambda=0.005$, for noise level 30 we set $\mu=0.005$, $\lambda= 0.005$, while for a noise level  50, we considered $\mu=0.00015$ and $\lambda=0.005$.

\begin{figure}[ht]
    \centering
\includegraphics[width=0.93\columnwidth]{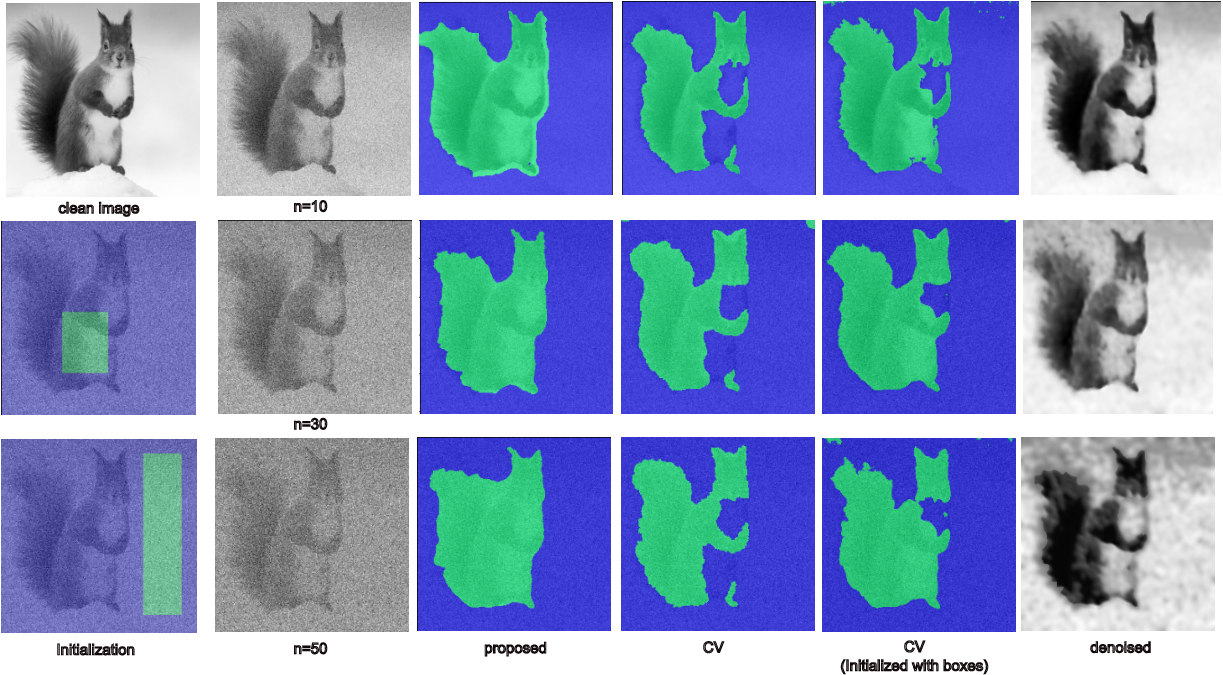}
    \caption{\textbf{Segmentation results obtained on the noisy images showing a squirrel corrupted with three different noise levels.} The first column shows clean input image, and initialisation for foreground and background regions, while in the second column the noisy versions of the given image are depicted. The remaining ones present the segmentation results obtained using the proposed strategy with the segmentation Algorithm~\ref{alg_acc}, the segmentation masks using the Chan-Vese algorithm provided by skimage~\cite{scikit-image} with checkerboard initialisation and box initialisation, respectively. The last column shows the denoised images which are obtained by composing the obtained segmentation mask and expert denoiser outputs. } 
    \label{squirrel}
\end{figure}

%\subsection{Texture Segmentation}
%Further, we apply our method on noisy texture images generated of textures from the USC-SIPI Image Database. 
%\begin{figure}
%\centering
%    \includegraphics[width=0.9\textwidth]{images/brodatz_noisy.svg}
%    \caption{Caption}
%\end{figure}
In the following we  discuss some possible extensions, and current limitations of the proposed joint denoising and segmentation approach.
\section{Extensions and limitations}
In this scetion, we first show how our unified framework can be extended to the (multichannel) multiclass segmentation case. Then we discuss the current limitations of the proposed joint denoising and segmentation approach.

%First, our proposed unified framework can be extended to the (multichannel) multiclass segmentation case, as we  discuss in the following paragraph.
\subsection{Vector-valued multi-class model}
In order to segment a noise-corrupted vector-valued image represented as $\boldsymbol{f} = (f_1,\dots,f_L)$ into  $C$ different regions, we can consider $C$ dedicated neural networks acting as denoising experts for each region. In this case, the objective is to estimate $C$ segmentation masks $\{u_i\}_{i=1}^C$ satisfying the simplex constraint, i.e. $\sum_{k=1}^C u_i = 1$, as well as the set of network parameters $\boldsymbol{W}^{\text{MC}} = (W^{\text{MC}}_1,\dots,W^{\text{MC}}_C)$. With these notations,  the energy~\eqref{eq:joint2} can  be extended to segment noise-corrupted, vector-valued images $\boldsymbol{f}$ as 
\begin{equation}\label{eq:joint_multiclass}
\begin{aligned}\mathcal{E}_{f,\la}(\boldsymbol{u},\boldsymbol{W}) 
\coloneqq i_{\mathbb{A}} (\boldsymbol{u})& + \lambda \lvert \boldsymbol{u}\rvert_{\text{TV}}+  \sum_ {i=1}^C\sum_{j=1}^L\int_{\Omega} \left(f_j-\Do_{W_i^{\text{MC}}}(f_j)\right)^2 u_i \,.\end{aligned}
\end{equation}
As before, it may not be necessary to train $C$ different denoising networks, as some regions may be assumed to be constant and in this case the ``expert'' for region $i$ can be replaced by the mean value of the image inside region $i$. 

\subsection{Limitations}
One limitation is encountered when training two denoisers, where one might encompass parts of both regions during segmentation, leading to undesired mask convergence. Although we propose a strategy to address this, adding a constraint to prioritize denoising performance in initial regions could be beneficial.

Moreover, the Noise2Fast denoiser may not be suitable for certain images, such as the zebra example, due to its local filter operation. This can result in outcomes similar to the piecewise constant Chan-Vese model. We modified the Noise2Fast strategy and suggest exploring alternative denoisers, such as the deep image prior~\cite{ulyanov2018deep}. Additionally, using contrastive learning ideas, introducing new data fitting terms focusing on structural similarities within classes could enhance segmentation performance.

\section{Conclusion}\label{sec:conclusion}
In this work, we have proposed a novel energy functional for the joint denoising and segmentation of images. Our framework combines the advantages of well-established variational models with modern self-supervised deep learning strategies. A major strength of the method lies in the fact that it can handle single images without the need of ground truth segmentation masks or noisy-clean training pairs. Further, the energy functional is designed in a such a way that both tasks benefit from each other, which has also been confirmed by experiments. %For future work, we plan to further improve the proposed method in order to address its current weaknesses.
\bibliographystyle{plain}
\bibliography{main}

% that's all folks
\end{document}